%% file: arxiv_version.tex
\documentclass{article}

\usepackage[preprint]{neurips_2026}

\usepackage[utf8]{inputenc} 
\usepackage[T1]{fontenc}    
\usepackage{url}            
\usepackage{amsfonts}       
\usepackage{nicefrac}       
\usepackage{microtype}      

\usepackage{afterpage}
\usepackage{enumerate}
\usepackage{bm}
\usepackage{multirow}
\usepackage{diagbox}
\usepackage{amsmath,amssymb,mathrsfs}
\usepackage{natbib}
\usepackage{bbm}
\usepackage{makecell}
\usepackage{upgreek}
\usepackage{mathtools}
\usepackage{xcolor}

\usepackage{dsfont,chngcntr}
\usepackage{smile}
\usepackage{stfloats,booktabs}
\usepackage{setspace}
\usepackage[english]{babel}
\usepackage{graphicx} 
\usepackage{wrapfig}
\usepackage{subcaption}
\usepackage{etoc}
\usepackage[colorlinks,
linkcolor=blue,
anchorcolor=blue,
citecolor=blue
]{hyperref}

\usepackage[capitalize,noabbrev]{cleveref}

\newtheorem{theorem}{Theorem}[section]

\newtheorem{corollary}[theorem]{Corollary}

\theoremstyle{definition}
\newtheorem{definition}[theorem]{Definition}
\newtheorem{assumption}[theorem]{Assumption}

\theoremstyle{remark}
\newtheorem{remark}[theorem]{Remark}

\def\tr{\mathop{\text{tr}}\kern.2ex}

\def\supp{\mathop{\text{supp}}}

\long\def\comment#1{}

\def\tr{\mathop{\text{Tr}}}

\newcommand{\bel}{\begin{eqnarray}\label}
\newcommand{\eel}{\end{eqnarray}}
\newcommand{\bes}{\begin{eqnarray*}}
	\newcommand{\ees}{\end{eqnarray*}}

\let\tilde\widetilde
\newlength{\tocrulewidth}
\def\mid{\,|\,}

\def\EE{{\mathbb E}}

\def\supp{\mathop{\text{supp}\kern.2ex}}

\def\supp{\mathop{\text{supp}}}

\def\tr{\mathrm{Tr}}

\title{Data-Constrained Language Model Pretraining: Improved Regularization and Scaling Laws}

\author{Zhiwei Xu\textsuperscript{$1$}, 
Shihao Wu\textsuperscript{$1$}, 
Hanseul Cho\textsuperscript{$2$}, 
Wei Hu\textsuperscript{$1$}\thanks{Equal advising.} , 
Yixin Wang\textsuperscript{$1$}$^\ast$
\\
\textsuperscript{$1$}University of Michigan,
\textsuperscript{$2$}KAIST AI
\\
\texttt{\{zhiweixu,wshihao,vvh,yixinw\}@umich.edu,jhs4015@kaist.ac.kr}
}
\date{}

\begin{document}

\maketitle

\etocsettocdepth.toc{none}

\begin{abstract}
Classical scaling laws for language model pretraining balance model size against training dataset size under a fixed compute budget, assuming abundant data and a single pass over the corpus. As training compute grows faster than the supply of natural language data, pretraining is likely to enter a data-constrained, compute-rich regime where models train for multiple epochs over a finite dataset. We study data-constrained pretraining along two axes, regularization and scaling. For regularization, we study masked-input regularization (MIR), an auxiliary next-token prediction loss on randomly masked inputs. MIR tests whether the random masking central to diffusion language models can benefit autoregressive pretraining without architectural changes or inference overhead. Across 72M to 1.4B parameter models, we find that MIR added on top of strong weight decay improves validation loss over autoregressive strong-weight-decay-only models, with downstream gains at 1.4B. For scaling, we propose SoftQ, a scaling law that couples model size and data size to capture their interaction under repeated data. Classical alternatives such as the Chinchilla law use an additive form that decouples these terms, making them misspecified in the data-constrained regime. We find that SoftQ fits data-constrained experiments substantially better than these alternatives, and estimates MIR's gains as equivalent to roughly 1.3 times as much unique training data. We release our code at \url{https://github.com/yixinw-lab/dc_pretrain}.
\end{abstract}

\section{Introduction}

Scaling laws \citep{kaplan2020scaling, hoffmann2022training} are widely used to choose model size and training-token budget for large language model pretraining. Classical scaling laws are largely compute-centric: they study how to allocate a fixed compute budget between parameters and tokens, assuming that unique training data can scale freely with compute. In this abundant-data setting, pretraining is typically performed with a single pass over a large corpus.

However, training compute is growing faster than the supply of natural language data \citep{villalobos2022, epoch2024, CommonCrawl2025}, 
making data-constrained, compute-rich pretraining increasingly important. In this regime, the unique dataset is fixed, and additional compute is spent on larger models and multiple passes over the same corpus. Prior work has begun to study this setting:
\citet{niklas2023} tuned data repetition while fixing weight decay to 0.1 and proposed scaling laws based on effective resources that saturate with repetitions and excess parameters; \citet{kim2026pretraining} further showed that large weight decay is critical for preventing overfitting.

This shift raises two linked questions. The first concerns regularization: how can models avoid overfitting when compute increases but unique data does not? Prior work points to strong weight decay as one answer. A second possibility comes from masked diffusion language models (dLLMs), which typically use the same transformer architecture as autoregressive (AR) models but train by predicting randomly masked tokens. Under identical hyperparameters, dLLMs achieve lower validation loss than AR transformers in the data-constrained regime \citep{ni2025diffusion, prabhudesai2025diffusion}, suggesting that random masking may itself act as a form of regularization. However, these comparisons do not isolate masking from regularization strength: the dLLM advantage may be complementary to strong weight decay, or it may largely reflect insufficiently strong regularization in the AR baseline. This motivates our first question: how do random masking and weight decay interact, and how much does each contribute on top of the other?

\begin{figure*}[t]
    \centering
    \begin{subfigure}[t]{0.46\textwidth}
        \centering
        \includegraphics[width=\linewidth]{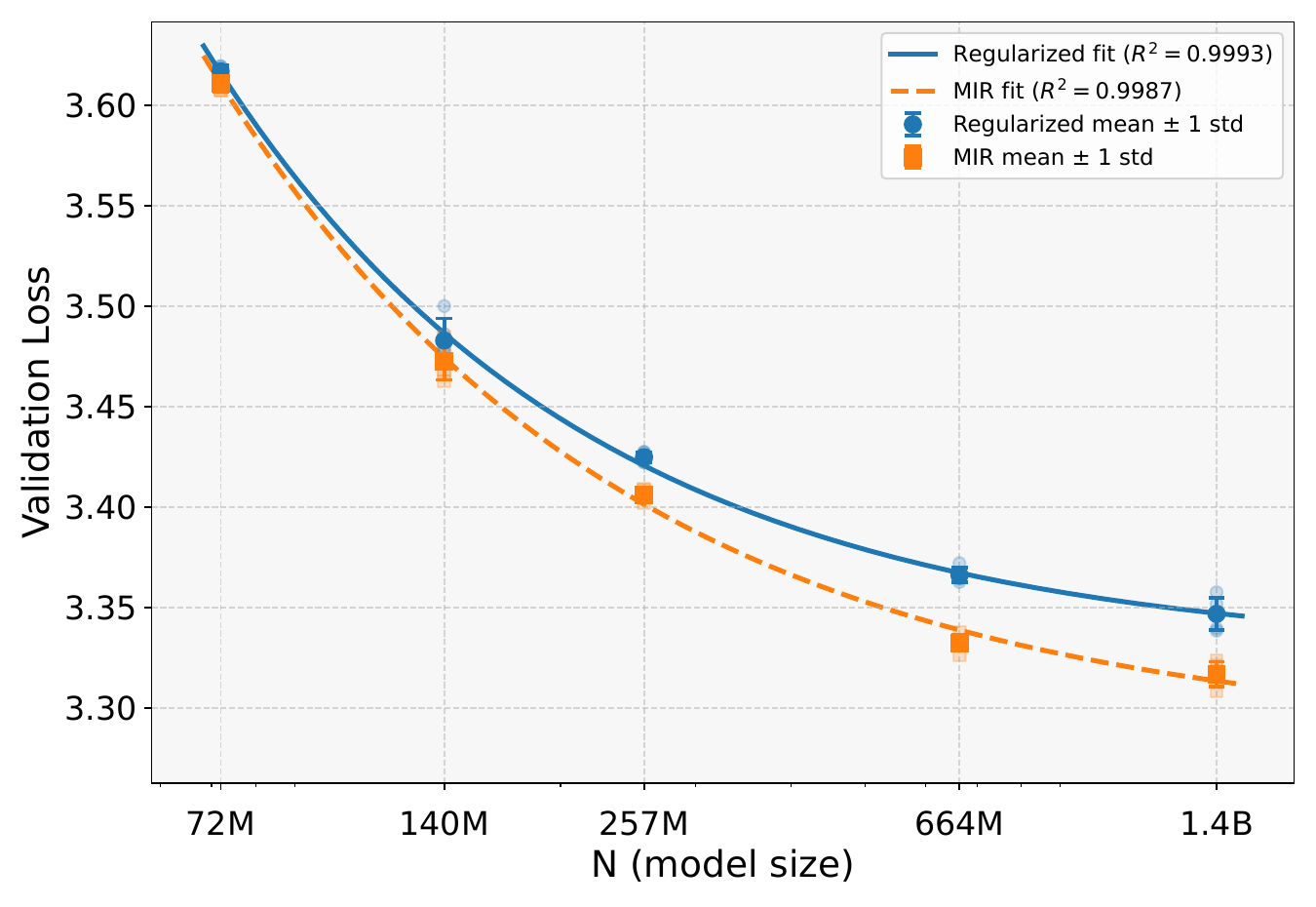}
        \caption{MIR improves a strong AR baseline.}
        \label{fig:dclm100m}
    \end{subfigure}\hfill
    \begin{subfigure}[t]{0.485\textwidth}
        \centering
        \includegraphics[width=\linewidth]{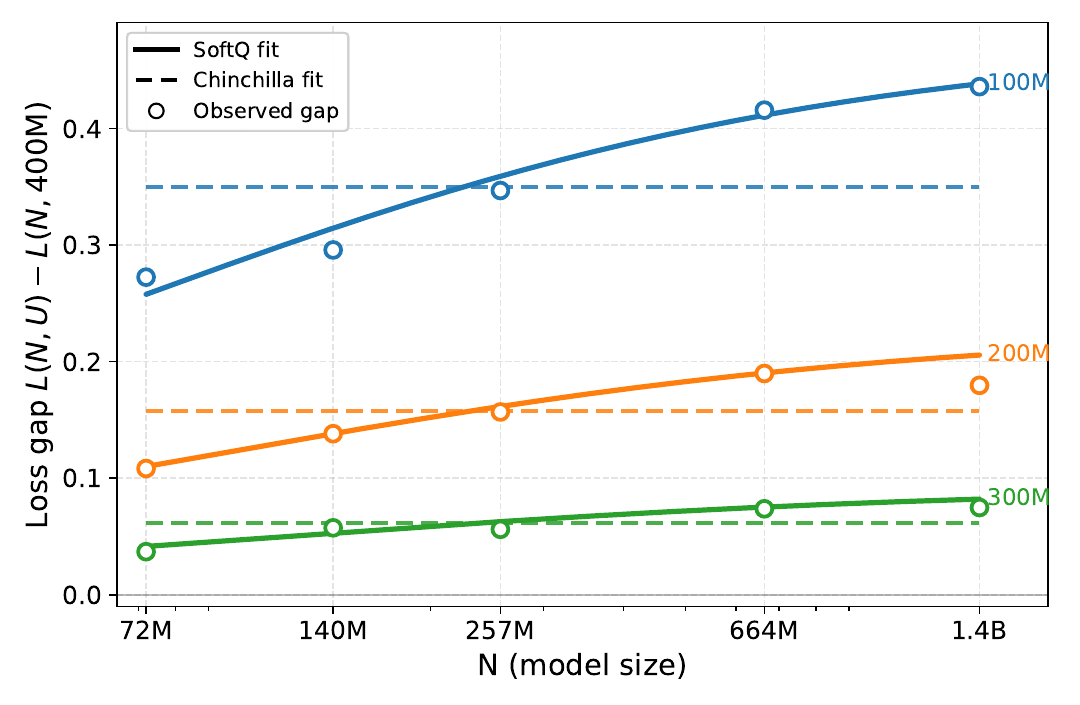}
        \caption{SoftQ captures data--model coupling.}
        \label{fig:softq-gap-overview}
    \end{subfigure}
    \caption{\textbf{Overview of the main results.}
    Left: On DataComp-LM (DCLM) dataset \citep{dclm24} with \(100\mathrm{M}\) unique training tokens, MIR improves validation loss over the
    strongly regularized autoregressive baseline across model sizes. Points
    show means over five random seeds, error bars show one standard deviation,
    and faint markers show individual runs. Right: On the
    strongly regularized baseline grid, we plot the loss gap
    \(L(N,U)-L(N,400\mathrm{M})\) for unique data budget
    \(U\in\{100\mathrm{M},200\mathrm{M},300\mathrm{M}\}\). Chinchilla predicts
    a model-size-invariant gap for each \(U\), while SoftQ tracks the empirical
    fan-out: the penalty from limited unique data grows with model size.}
    \label{fig:overview}
    \vspace{-10pt}
\end{figure*}

The second question concerns scaling: what loss law describes the data-constrained, compute-rich regime? Chinchilla-style laws were fit to single-pass, abundant-data training and may not capture the validation-loss surface when unique data, rather than compute, is the binding resource. In particular, their additive form predicts that the loss gap between two unique-data budgets should be independent of model size. In this paper, we study both questions in the data-constrained, compute-rich regime.

\textbf{Finding 1: Random masking provides regularization complementary to strong weight decay.} We first ask how the two regularization mechanisms interact. We find that strong weight decay is not specific to AR pretraining: applying the AR-tuned weight decay to dLLMs substantially lowers their validation loss, and once both models are strongly regularized, their validation losses become comparable across the model sizes we study. Given that strong weight decay alone provides such substantial regularization, this makes it unclear whether random masking can still provide additional benefit once strong weight decay is already in use.

To isolate this effect, we study \emph{masked-input regularization} (MIR), a minimal modification to standard AR pretraining. Let $x$ denote a clean sequence and $\tilde{x}$ a randomly masked version of the same sequence. Instead of optimizing only the standard next-token prediction loss $\mathcal{L}_{\mathrm{NTP}}(x)$, MIR optimizes
\[
\mathcal{L} = \mathcal{L}_{\mathrm{NTP}}(x) + \lambda \mathcal{L}_{\mathrm{NTP}}(\tilde{x}).
\]
Thus, the model trains on both clean and masked inputs, using the masked-input loss as an auxiliary regularizer. MIR requires no architectural changes and preserves standard autoregressive decoding at inference. Although it increases training compute, our setting is data-constrained and compute-rich, so we study MIR as a way to improve loss at a fixed unique-data budget, i.e., data efficiency rather than compute efficiency.

Across models from 72M to 1.4B parameters trained on DCLM \citep{dclm24} and Stack-V2 \citep{stackv224}, MIR consistently improves validation loss on top of strong weight decay (Figure~\ref{fig:dclm100m}). At 1.4B parameters, it also yields substantial downstream gains, including +10.2 points on BoolQ and +2.2 points on SciQ.

\textbf{Finding 2: Chinchilla is misspecified in the data-constrained, compute-rich regime; a coupled scaling law fits better.}
To quantify how much unique data MIR is worth, we extend our experiments across five model sizes and four unique-data budgets and fit several scaling laws. The additive Chinchilla form \citep{hoffmann2022training} fits poorly in this regime: it predicts that the validation-loss gap between two data budgets is independent of model size, whereas our experiments show that this gap grows with model size (Figure~\ref{fig:softq-gap-overview}).

We propose the \emph{SoftQ scaling law}, a five-parameter form that couples model size and data size through a soft bottleneck motivated by the skill-learning view of scaling laws \citep{quanta-scalinglaw}. SoftQ achieves better in-sample fit and out-of-sample prediction than Chinchilla, Quanta \citep{quanta-scalinglaw}, and Muennighoff-style \citep{niklas2023} laws on our dataset. The same ranking holds on an independent dataset from \citet{kim2026pretraining}. Using SoftQ as the baseline scaling law, we estimate MIR's gain over the strongly regularized baseline to be equivalent to roughly $1.3\times$ as much unique training data at the 200M--400M token budgets.

\textbf{Contributions.} We summarize our contributions as follows: (i) We show that large
weight decay substantially improves dLLMs in the data-constrained
regime, and that random masking further improves strongly
regularized AR models. Building on this observation, we propose MIR, a minimal
recipe that augments strongly regularized AR pretraining with an
auxiliary masked-input next-token loss; we estimate MIR to be worth
roughly $1.3\times$ as much unique training data at the 200M to 400M
token budgets. (ii) We show that additive Chinchilla-style scaling
laws do not fit the data-constrained, compute-rich regime, and propose
SoftQ, a five-parameter scaling law that couples model and data size
and substantially outperforms these alternatives.

\section{Setup: Data-Constrained Autoregressive and Masked Pretraining}
\subsection{Data-Constrained and Compute-Rich Pretraining}
\label{subsec:background_pretrain}

Let \(N\) denote the number of model parameters, \(U\) the number of unique
pretraining tokens, \(N_E\) the number of epochs over those tokens, and
\(D = U N_E\) the total number of training tokens. For a standard dense decoder-only transformer trained with next-token prediction, the
training compute is approximately
\(
C(N,D) \approx 6ND.
\)

Classical compute-optimal scaling laws \citep{kaplan2020scaling,hoffmann2022training} model evaluation loss as a function of
model size and training-token budget. In the abundant-data regime, the processed
tokens can be treated as fresh samples, so the distinction between unique tokens
and repeated tokens is not explicit. The standard compute-allocation problem is
\[
(N^\star(C),D^\star(C))
=
\arg\min_{N,D} L_{\mathrm{eval}}(N,D)
\quad
\mathrm{s.t.}
\quad
C(N,D)=C .
\]
For example, Chinchilla-style parametric scaling writes
\(
\widehat L(N,D)
=
E + A N^{-\alpha} + B D^{-\beta},
\)
and then chooses the point on this surface that minimizes loss under the
training-compute constraint. Such laws are highly effective when new data is
available, but they do not distinguish a token budget \(D\) consisting of fresh
tokens from the same budget obtained by repeatedly training on a finite corpus.

In data-constrained, compute-rich pretraining, the unique-token budget \(U\) is fixed or
bounded, and \(C\) is unbounded. Additional training compute can be spent by increasing the number
of epochs, increasing model size, or changing regularization. 
Prior work studies
several versions of this problem. \cite{niklas2023} model repeated data under
compute constraints by replacing raw token and parameter counts with effective
resources that saturate as repetitions and excess parameters grow. \cite{kim2026pretraining}
study a more compute-rich setting in which the unique data is fixed and the
training recipe is tuned to estimate the best attainable loss at each model
scale.

We follow the compute-rich perspective. For a fixed architecture family,
optimizer class, data distribution, and evaluation protocol, define the
optimized validation-loss envelope
\[
L^\star(N,U)
=
\inf_{h\in\mathcal H}
L_{\mathrm{eval}}(N,U;h),
\]
where \(h\) includes the tunable training hyperparameters, such as the number of
epochs, learning-rate schedule, weight decay, and other regularization choices.
In this formulation, \(D=U N_E(h)\) determines the compute used by a particular
training run, but compute is not the binding constraint used to define
\(L^\star\). The goal is therefore to model the joint dependence of the
best-achievable loss on model size \(N\) and unique data size \(U\).

\vspace{-10pt}

\subsection{Autoregressive and Masked Diffusion Language Models}

Let \(p_{\theta}\) denote the transformer model and \(\{x_i\}_{i=1}^n\) the training dataset, 
where each sample \(x_i = [x_{i,0}, x_{i,1}, \dots, x_{i,T-1}]\) is a sequence of length \(T\).
Autoregressive models predict tokens from left to right. The training objective \(\mathcal{L}_{\mathrm{NTP}}\) is
\(
-\sum_{i=1}^n \sum_{t=0}^{T-1} \log p_{\theta}(x_{i, t} \mid x_{i, < t}) / (nT).
\)
For each sequence \(x_i\), dLLMs sample a mask ratio \(r_i \sim \text{Unif} (0, 1]\), 
and use a Bernoulli random variable \(\text{Bern}(r_i)\) to decide whether to mask the token \(x_{i,t}\) or not 
for each position \(t \in [0, T)\). The model only predicts the true tokens at those masked positions.
The training objective is
\[
-\frac{1}{n T}\sum_{i=1}^n\Big[\frac{1}{r_i}\sum_{t=0}^{T-1} \mathbb{I} (\tilde{x}_{i,t} = \text{MASK}) \log p_{\theta}(x_{i,t} \mid \tilde{x}_i) \Big],
\]
where \(\tilde{x}_i\) represents the masked sample \(x_i\).

\begin{wrapfigure}[18]{R}{0.46\textwidth}
    \centering
    \vspace{-8pt}
    \includegraphics[width=\linewidth]{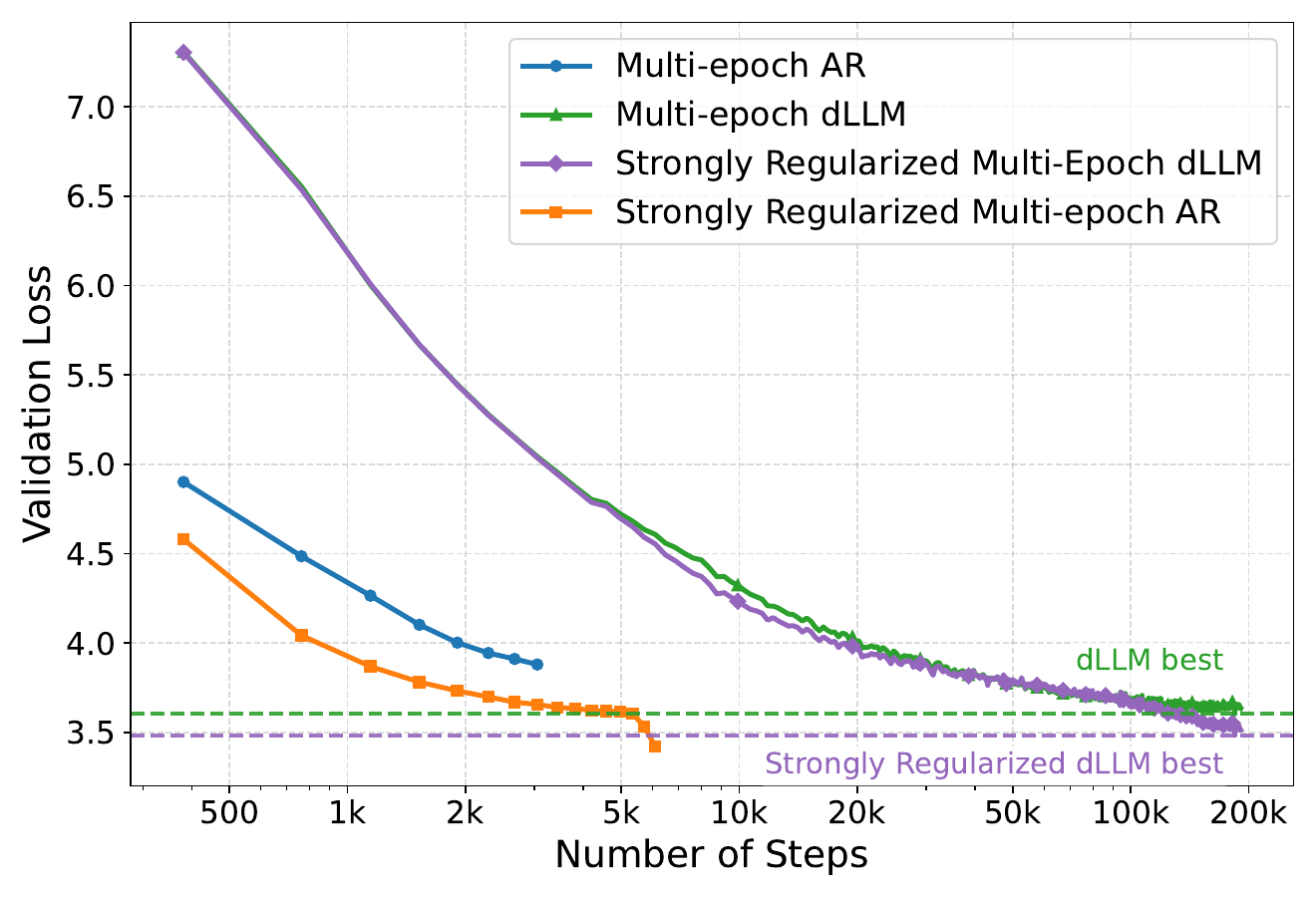}
    \caption{Validation Loss dynamics on DCLM 100M for the 257M model. Large weight decay substantially improves both multi-epoch AR and dLLM training; with both well regularized, their validation losses become comparable.}
    \label{fig:dllm_vs_AR_vs_strong_regularized}
\end{wrapfigure}

\section{Regularization in the Data-Constrained, Compute-Rich Regime}
\label{sec:mir}
\subsection{Weight Decay Transfers Across AR and dLLM Pretraining}
Recent studies report that dLLMs outperform AR models in the data-constrained regime \citep{ni2025diffusion, prabhudesai2025diffusion}, using weight decay \(\mathrm{wd}=0.1\) for both. Independently, \citet{kim2026pretraining} showed that large weight decay is critical for AR pretraining in this regime. We ask whether this benefit transfers to dLLMs and re-examine the AR–dLLM comparison under matched large-weight-decay treatment. On DCLM with 100M unique tokens, we compare four recipes at three model sizes (140M, 257M, and 664M): (i) Multi-epoch AR (\(\mathrm{wd}=0.1\), tuned epochs); (ii) Multi-epoch dLLM (\(\mathrm{wd}=0.1\), per \citet{prabhudesai2025diffusion}); (iii) Strongly Regularized AR with epochs, learning rate, and weight decay jointly tuned following \citet{kim2026pretraining}; (iv) Strongly Regularized dLLM, which inherits the AR-tuned weight decay but keeps other hyperparameters at \citet{prabhudesai2025diffusion} defaults. We report final-step validation loss for AR and the best across-epoch loss for dLLM.

Figure~\ref{fig:dllm_vs_AR_vs_strong_regularized} shows all four recipes at 257M. With \(\mathrm{wd}=0.1\), we reproduce the finding that dLLM (\(3.60\)) outperforms multi-epoch AR (\(3.88\)). 
Large weight decay dramatically improves both: it reduces AR loss to \(3.42\) and, when ported to dLLM, reduces dLLM loss to \(3.48\). 
Since dLLM validation loss is the negative evidence lower bound (an
upper bound on the negative log-likelihood) while AR loss is exact negative
log-likelihood, the slightly higher dLLM validation loss does not imply worse
performance. The two strongly regularized recipes have losses comparable at 140M, 257M, and 664M
(see Table~\ref{tab:dclm100m_dllm_comparison} in Appendix~\ref{app:exp_details}), implying that the previously reported
AR–dLLM gap is largely explained by insufficient AR regularization. Still, the fact that
dLLMs avoid the repeated-epoch collapse seen in weakly regularized AR
suggests that random input masking acts as an implicit regularizer in its own
right. We next ask whether this masking signal can contribute additional gains
on top of strong weight decay when added to AR training.

\subsection{Masked Input Regularization}
To capture this hypothesized benefit without abandoning the efficiency of standard AR decoding,
we study masked-input regularization (MIR). The method samples a mask ratio \(r\) from a uniform distribution 
\(\text{Unif}(r_{\min}, r_{\max})\) for each input sequence \(x\). 
At each position \(t \in [0, T-1]\), a Bernoulli random variable with success probability \(r\) determines 
whether to replace the token \(x_t\) with a specialized [MASK] token. 
Let \(\tilde{x}\) denote this corrupted sequence. Without altering the model architecture, MIR adds an auxiliary next-token prediction loss on the masked sequence:
\[
\mathcal{L} = \mathcal{L}_{\mathrm{NTP}}(x) + \lambda \mathcal{L}_{\mathrm{NTP}}(\tilde{x}).
\]
It requires two forward passes to calculate the training loss for each batch.
MIR therefore increases per-step training compute. Because our focus is the data-constrained, compute-rich regime, we use MIR to study whether additional compute can improve loss at a fixed unique-data budget, rather than as a compute-efficiency method. See tuning details and regularization coefficient in Appendix \ref{subsec:mir_coeff_tuning}.
\begin{remark}
We study the data-constrained, compute-rich regime defined in Section~\ref{subsec:background_pretrain}, where unique data rather than compute is the binding resource, so we focus on improving data efficiency rather than compute efficiency. We further quantify MIR's gain in data efficiency in Section~\ref{sec:mir-data-efficiency}.
\end{remark}

\subsection{Theoretical Intuition: Reducing Memorization via Masking}\label{sec:theo:intuition}

We provide intuition for how masking improves validation loss by analyzing a toy context-specific noise model in
Appendix \ref{sec:theory}. This model decomposes each sequence into three parts:
a context-specific component that enables memorization and acts as noise for
generalization, a generalizable component that 
contains predictive features, and an output token to be predicted from the first
two components. Under the data-constrained, compute-rich regime, we establish the
following dynamic.

\textbf{Theorem (Informal).} \textit{Under the context-specific noise model in the
data-constrained, compute-rich regime, standard autoregressive pretraining can
minimize training loss by relying almost entirely on the context-specific component, thereby
memorizing patterns that do not generalize to unseen examples. In contrast, MIR
regularizes the model's dependence on the context-specific components and
encourages it to learn predictive patterns on the generalizable components, 
thereby strictly improving validation loss. Moreover, for a fixed data size,
this improvement increases as model capacity grows.}

This informal theorem illustrates that MIR improves validation loss by reducing
the model's dependence on context-specific noise and encouraging it to learn
generalizable predictive features. We provide the formal model definition, assumptions,
theorem statement, and proofs in Appendix \ref{sec:theory}.

\subsection{Empirical Results}
\label{sec:mir_exp}

We evaluate MIR against the strongly regularized AR model baseline along four axes: scaling behavior on natural language, whether the gain transfers to coding data, where the gain comes from, and whether it translates to downstream tasks. Throughout, the only difference between MIR and the baseline is the auxiliary masked-input loss;
architecture, optimizer, and the per-cell-tuned
$(\text{epochs}, \text{weight decay}, \text{learning rate})$ configuration are
held fixed across the two recipes.

We find the following.
(1) On DCLM with 100M unique tokens, MIR reduces validation loss at every model
scale from 72M to 1.4B and on every matched random seed, with the average gain
growing from roughly $0.006$ at 72M to about $0.03$ at 1.4B. This trend is consistent with
the theoretical prediction in Section~\ref{sec:theo:intuition} that
overparameterized models benefit more from masking-based regularization.
(2) The benefit is not specific to natural language: with hyperparameters tuned
only on DCLM, MIR also reduces validation loss at all five model sizes on the
code-heavy Stack-V2 dataset.
(3) A token-level analysis on the 1.4B model shows that MIR's gain comes
from a broad set of validation positions rather than a few outliers. At the
positions where MIR most outperforms the baseline, the true next token is
itself usually a common one such as a function word or punctuation, and what
makes prediction difficult is the preceding context: rare names, mixed scripts,
broken word pieces, or noisy web text (Section~\ref{subsec:token_gap_main}).
(4) The loss improvement directionally transfers to downstream
tasks: the 1.4B MIR model outperforms the strongly regularized baseline on six
of eight zero-shot metrics, including $+10.2$ points on BoolQ and $+2.2$ points
on SciQ.

\textbf{Experimental setup.} To evaluate masked-input regularization, we train models on two distinct data distributions: 
standard natural language from DCLM \citep{dclm24} and code-heavy text from Stack-V2 \citep{stackv224}. 
For both datasets, the pretraining budget is fixed to 100M unique seed tokens and 10M tokens are reserved for validation.
We tune hyperparameters only on DCLM data and test whether the improvement from MIR still exists on the Stack-V2 dataset.

We build a scaling ladder with five model sizes:
\[
\text{ScalingLadder}(k) = (kW_1, kL_1, S_1, B_1),
\]
where \(W_1=1024\) is the embedding dimension when \(k=1\), \(L_1=12\) 
is the number of layers when \(k=1\), \(S_1 = 2048\) is the sequence length, 
\(B_1 = 128\) is the total batch size, and  $k \in \{0.5, 0.75, 1, 1.5, 2\}$. 
Across the scaling ladder, the attention head dimension is fixed at 64, while the depth, embedding dimension, MLP
dimension, and number of attention heads increase with scale. The model size ranges from \(72\)M to \(1.4\)B.
We use a Llama-style decoder-only transformer and use the same model architecture for all experiments. AdamW optimizer is used for all experiments. 
We use grid search to select the number of training steps, learning rate, and weight decay for each model.
See Appendix \ref{app:exp_details} for details on the optimizer, model architecture, and hyperparameter search.

\textbf{Validation loss improvements.} Figure~\ref{fig:dclm100m} visualizes validation loss across the scaling ladder for the DCLM 100M dataset, averaged over five random seeds. 
MIR improves validation loss over the strongly regularized baseline for every matched seed at every model scale. 
On the 1.4B parameter model, for example, MIR reduces the mean validation loss from 3.347 to 3.317. The average gain grows from roughly 0.006 loss at 72M parameters to about 0.03 loss for the two largest models, suggesting that MIR is especially useful when model
capacity is high relative to the amount of unique training data. This trend is qualitatively consistent with our theoretical analysis, which predicts that larger overparameterized models are more prone to overfitting, and therefore benefit more from masking-based
regularization.

Crucially, this regularization benefit generalizes beyond standard natural language. 
We repeat the same 100M token experiments to evaluate performance on code-heavy data: 
on Stack-V2, MIR reduces validation loss at all five model sizes, with absolute gains from 0.008 to 0.020 loss; see full numbers in Table \ref{tab:stack_v2_losses} in the Appendix.

\textbf{Where MIR helps: Token-level analysis.}
\label{subsec:token_gap_main}
To localize where the validation-loss gain comes from, we compare the 1.4B regularized baseline and the 1.4B MIR model on the \(10\)M DCLM eval dataset. 
For each position \(t\), we compute the negative log-likelihood on the true target \(y_t=x_{t+1}\) and define the token-level loss gap as
\(
\Delta \ell_t = \ell_{\mathrm{base}}(t) - \ell_{\mathrm{MIR}}(t),
\)
so that positive values favor MIR. Figure~\ref{fig:loss_gap_tail_overlay_mir_example_comparison} Left shows that the MIR-better tail 
is both larger and slightly heavier than the baseline-better tail after removing the center region \(|\Delta \ell_t| < 1\): 
\(6.61\%\) of tokens satisfy \(\Delta \ell_t \ge 1\), while \(5.41\%\) satisfy \(\Delta \ell_t \le -1\). 
Therefore, the overall loss gain is not driven by a few isolated outliers, but appears on a broad set of hard validation tokens.

The top positive-gap tokens reveal a clear qualitative pattern. We rank all validation positions by \(\Delta \ell_t\), 
decode the top \(0.1\%\) MIR-better positions, and inspect the true token together with the preceding and following tokens. 
These high-gap examples are dominated by continuation problems rather than standalone rare targets: 
\(62.6\%\) are word or subword continuations, \(16.3\%\) occur in non-English or transliterated text, and \(11.6\%\) are punctuation tokens. 
Importantly, the true token is often a common token such as ``and'', ``to'', ``of'', ``is'', a comma, 
or a closing parenthesis. What makes these positions hard to predict is the local prefix context: non-English languages, rare names, mixed scripts, 
broken word pieces, or noisy web and markup text.

\begin{figure*}[ht]
    \centering
    \begin{subfigure}[t]{0.41\textwidth}
        \vspace{0pt}
        \centering
        \includegraphics[width=\linewidth]{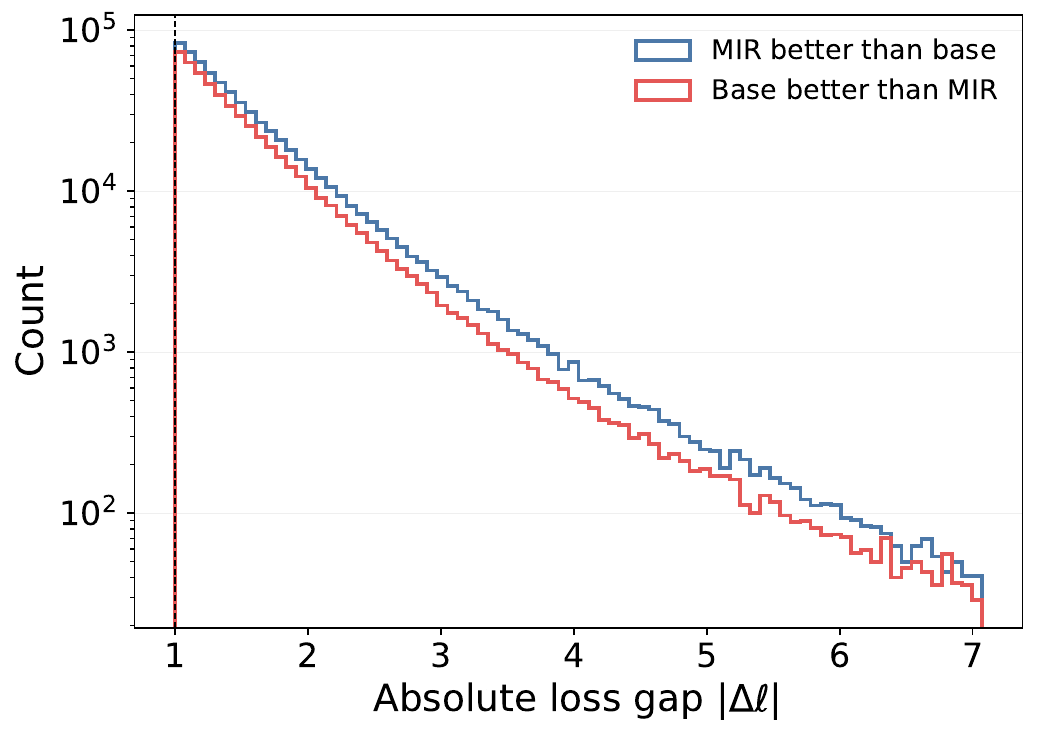}
    \end{subfigure}\hfill
    \begin{subfigure}[t]{0.56\textwidth}
        \vspace{0pt}
        \centering
        \includegraphics[width=\linewidth]{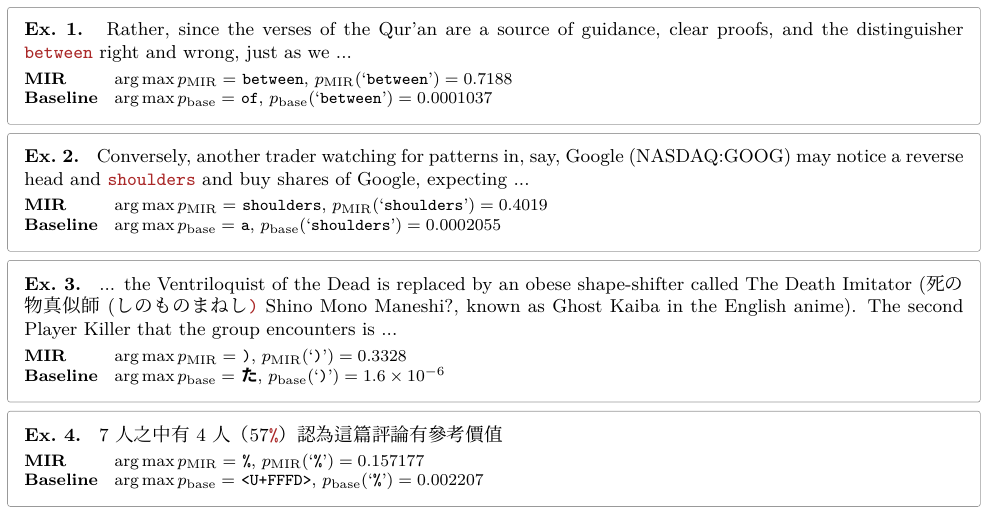}
    \end{subfigure}
    \caption{Left: Absolute token-level loss-gap tails on all validation tokens 
    for the 1.4B models after removing the center region \(|\Delta \ell| < 1\). 
    The positive tail, where MIR assigns higher probability 
    to the true next token than the strongly regularized baseline, is both larger and slightly heavier.
    Right: Representative MIR-better tokens from the top \(0.1\%\) positive-gap set. 
    In each example, the target next token is highlighted in red, and the probabilities assigned to that token by MIR and by the strongly regularized baseline are shown below. Many large-gap cases involve names, 
    subword completions, mixed scripts, or noisy web and technical text, even when the true token itself is common.}
    \label{fig:loss_gap_tail_overlay_mir_example_comparison}
\vspace{-10pt}
\end{figure*}

Representative examples illustrate how these gains arise in practice. Figure~\ref{fig:loss_gap_tail_overlay_mir_example_comparison}(right) shows cases 
where the baseline falls back to a generic continuation or keeps following the wrong local pattern, whereas MIR recovers the intended continuation. 
One example is a mixed-script entity name followed by a Japanese parenthetical gloss, where the correct next token is the closing parenthesis ``)''; 
MIR predicts it correctly, while the baseline keeps extending the Japanese string. 
Taken together, these results suggest that MIR helps most when the next-token decision depends on robustness to 
unusual or noisy local prefix context rather than on simple frequency-based continuation. This pattern is consistent with our theoretical intuition that masking 
regularizes the model's dependence on irrelevant details in the prefix context and encourages it to learn predictive features that generalize across contexts.

\textbf{Downstream evaluations.} To understand whether the improved validation loss translates to capability 
gains on downstream tasks, we evaluate the two 1.4B models trained 
on DCLM dataset with \(U=\) 100M across a suite of downstream tasks using \texttt{lm-evaluation-harness} \citep{eval-harness}. 

Table \ref{tab:downstream_evals} shows that the MIR-trained model 
achieves superior performance on six of the eight evaluated metrics. 
The improvements are particularly pronounced on reasoning and 
reading comprehension tasks, pushing accuracy on BoolQ up by 10.18 percentage 
points and SciQ up by 2.20 percentage points compared to the strongly regularized baseline.
Because these models are trained at academic scale, with 1.4B parameters and 100M unique training tokens, 
we view the downstream evaluations as a coarse capability check: MIR shows a large gain on BoolQ 
and smaller mixed changes elsewhere, with the overall pattern directionally agreeing with 
its validation-loss improvement.

\vspace{-10pt}
\begin{table}[h]
\centering
\caption{Downstream zero-shot evaluation for 1.4B models trained on the DCLM data with \(U=\)100M.}
\label{tab:downstream_evals}
\begin{tabular}{lrrr}
\toprule
Task & Random Guess & Regularized Baseline & + MIR \\
\midrule
ARC-Easy (acc\_norm) & 0.2500 & 0.3805 $\pm$ 0.0100 & \textbf{0.3893 $\pm$ 0.0100} \\
BoolQ (acc) & 0.5000 & 0.4511 $\pm$ 0.0087 & \textbf{0.5529 $\pm$ 0.0087} \\
HellaSwag (acc\_norm) & 0.2500 & 0.2833 $\pm$ 0.0045 & \textbf{0.2855 $\pm$ 0.0045} \\
PiQA (acc\_norm) & 0.5000 & \textbf{0.5996 $\pm$ 0.0114} & 0.5985 $\pm$ 0.0114 \\
RACE (acc) & 0.2500 & 0.2689 $\pm$ 0.0137 & \textbf{0.2766 $\pm$ 0.0138} \\
SciQ (acc\_norm) & 0.2500 & 0.5780 $\pm$ 0.0156 & \textbf{0.6000 $\pm$ 0.0155} \\
Lambada (acc) & $\sim$0.0000 & \textbf{0.2271 $\pm$ 0.0058} & 0.2261 $\pm$ 0.0058 \\
Lambada (perplexity) & N/A & 112.8966 $\pm$ 4.9091 & \textbf{106.7115 $\pm$ 4.5752} \\
\bottomrule
\end{tabular}
\end{table}

\section{Scaling in the Data-Constrained, Compute-Rich Regime}
\label{sec:scaling}
In this section, we extend the experiments to a five-by-four grid of model sizes by unique-data budgets and use it to (a) show that the classical Chinchilla scaling law is misspecified in the data-constrained, compute-rich regime, (b) propose the SoftQ scaling law as a better-fitting alternative, and (c) quantify MIR's data efficiency gain over the strongly regularized baseline.

\textbf{Constructing the baseline grid.} We choose three additional unique-data budgets beyond the 100M used in Section~\ref{sec:mir_exp}: 200M, 300M, and 400M. The grid is thus five model sizes \(\times\) four data sizes: \( \{72\text{M}, 140\text{M}, 257\text{M}, 664\text{M}, 1.4\text{B}\} \times \{100\text{M}, 200\text{M}, 300\text{M}, 400\text{M}\} \). For each cell, we tune the number of epochs, weight decay, and learning rate; the optimal weight decay is consistently much larger than the standard value of 0.1, so following \citet{kim2026pretraining} we call this the strongly regularized recipe. See Appendix~\ref{app:exp_details} for the hyperparameter search and best configurations. The result is a baseline dataset \(\{(N, U, L)\}\) of 20 points, where \(L\) is the validation loss of the AR model of size \(N\) in the scaling ladder trained on \(U\) unique tokens with the best hyperparameters \((N_E, \text{weight decay}, \text{learning rate})\) for that cell.

\vspace{-10pt}
\subsection{The SoftQ scaling law}

\textbf{Why Chinchilla is Misspecified.}
The Chinchilla scaling law decomposes loss into irreducible entropy, 
finite-parameter error, and finite-data error:
\begin{equation}
    L_{\mathrm{Ch}}(N,U)=E+\frac{A}{N^\alpha}+\frac{B}{U^\beta}.
    \label{eq:chinchilla-main}
\end{equation}
Its additive structure implies that the parameter and 
data terms are separable. Consequently, given a model with size \(N\), 
the loss gap between two unique data budgets \(U_1, U_2\) does not depend on \(N\):
\[
L_{\mathrm{Ch}}(N,U_1)-L_{\mathrm{Ch}}(N,U_2)=\frac{B}{U_1^{\beta}}-\frac{B}{U_2^{\beta}}.
\]
This prediction is at odds with the expected behavior in data-constrained,
compute-rich pretraining. The marginal value of additional unique data should
depend on model size. For sufficiently small models, both \(U_1\) and \(U_2\)
provide more unique information than the model can effectively exploit, so the
losses obtained from the two data budgets should be similar. In this regime, the
loss gap should be close to zero. For sufficiently large models, capacity is no
longer the binding constraint, and the difference in available unique information
between \(U_1\) and \(U_2\) should become visible in validation loss. The gap
should therefore increase with \(N\), reflecting a coupling between model size and
unique data budget that Chinchilla's additive form cannot represent.

We verify this behavior empirically. Figure~\ref{fig:softq-gap-overview} shows the diagnostic directly: the loss gap between each smaller data budget and the 400M budget increases with model size, but Chinchilla predicts a constant gap for each budget.
This motivates a coupled law rather than an additive one.

\textbf{Existing coupled laws.} Two prior laws have moved in this direction.
\cite{niklas2023} generalize the Chinchilla law by
replacing raw data and parameter counts with effective model size \(N'\) and effective data size \(D'\)
that saturate under repeated data and excess parameters.
It includes the number of epochs \(N_E\) as an additional input to predict the validation loss:
\begin{align}
    L_{\mathrm{M}}(N,U,N_E)
    =E+\frac{A}{(N')^\alpha}+\frac{B}{(D')^\beta},\quad   D' = f(U, N_E), \ N'= g(N, U, N_E),
\end{align}
which has seven parameters to fit.
\cite{quanta-scalinglaw} derive a scaling law from the quanta-skill learning model. They assume that the use frequencies of skills follow a power law and obtain
\(L(N,U)-E \propto n(N, U)^{-\alpha}\), where \(n(N, U)\) is the number of
skills the model can learn given \(N\) parameters and \(U\) unique tokens.
Under further assumptions, they show \(n(N, U) \propto N\) when
\(U \rightarrow \infty\) and \(n(N, U) \propto U^{1/(1+\alpha)}\) when
\(N \rightarrow \infty\). Concurrently, \cite{merrill2026olmo} proposed Expressivity-Aware Scaling Laws, which derived the same scaling properties.
Setting
\(n(N, U) = \big(A/N + B/U^{1/(1+\alpha)}\big)^{-1}\) yields the Quanta scaling
law:
\begin{equation}
    L_{\mathrm{Q}}(N,U)
    =
    E+
    \left(
        \frac{A}{N}
        +
        \frac{B}{U^{1/(1+\alpha)}}
    \right)^{\alpha},
    \label{eq:quanta-main}
\end{equation}
where the marginal value of increasing model size depends on the available data through the outer exponent.
We give the full expression of Muennighoff law in Appendix~\ref{app:scaling-law-details} and the detailed Quanta derivation in Appendix~\ref{app:quanta-derivation}.

\textbf{SoftQ.} 
 Motivated by the skill-learning view of scaling, we propose the \emph{SoftQ scaling law}, 
a soft-quanta law that combines the parameter-limited and data-limited regimes through a smooth bottleneck:
\begin{equation}
    L_{\mathrm{SoftQ}}(N,U)
    =
    E+
    \left(
        \frac{A}{N^{\rho}}
        +
        \frac{B}{U^{\rho/(1+\alpha)}}
    \right)^{\alpha/\rho}.
    \label{eq:softq-main}
\end{equation}
The parameter $\rho$ controls the sharpness of the transition between the parameter-limited and data-limited regimes.
As $U\rightarrow\infty$, the law recovers a parameter-scaling limit $L-E\propto N^{-\alpha}$; 
as $N\rightarrow\infty$, it recovers a data-scaling limit $L-E\propto U^{-\alpha/(1+\alpha)}$.
It has five fitted parameters, $\{A,B,E,\alpha,\rho\}$, matching the Chinchilla parameter count 
while explicitly coupling model size and data size.
When \(\rho = 1\), SoftQ reduces to the Quanta law, so SoftQ strictly nests Quanta as a special case while adding one parameter that controls the bottleneck sharpness.

\vspace{-5pt}
\subsection{Scaling Laws Comparison and MIR data efficiency}\label{sec:mir-data-efficiency}
\vspace{-5pt}
We compare Chinchilla, Quanta, Muennighoff, and SoftQ on three diagnostics: (1) full fit on the strongly regularized baseline results; 
(2) held-out fit, training on the 100M/200M/300M points and predicting the five 400M points; and 
(3) full fit on an independent baseline dataset provided by \citet{kim2026pretraining}.
For the fitting protocol, Chinchilla, Quanta, and SoftQ use the Approach-3-style objective of
\citet{hoffmann2022training}: Huber loss with threshold $\delta=10^{-3}$ on
log-loss residuals.  For the Muennighoff-style law, our main comparison uses a
dataset-adapted two-stage protocol: fit the base Chinchilla coefficients on the
same split, then hold them fixed while fitting only the decay constants
$R_N^\star$ and $R_D^\star$.  We report RMSE and MAE on the raw validation-loss
scale, and an SSE-based Gaussian AIC: \(n\log(\mathrm{RSS}/n)+2k,
\)
where $k$ is the number of fitted parameters.

\begin{table}[t]
\centering
\caption{Scaling laws comparison results. Lower is better.}
\label{tab:scaling-law-selection-main}
\small
\setlength{\tabcolsep}{3.2pt}
\begin{tabular}{lrrrrrrrr}
\toprule
& & \multicolumn{3}{c}{Full fit} &
\multicolumn{2}{c}{Held-out $400$M} &
\multicolumn{2}{c}{[Kim et al.] Full fit} \\
\cmidrule(lr){3-5}\cmidrule(lr){6-7}\cmidrule(lr){8-9}
Law & $k$ & RMSE & MAE & AIC & RMSE & MAE & RMSE & AIC \\
\midrule
Chinchilla  & 5 & 0.02653 & 0.01802 & -135.18 & 0.03106 & 0.02540 & 0.04041 & -92.68 \\
Quanta      & 4 & 0.01252 & 0.00889 & -167.23 & 0.01497 & 0.01207 & 0.02375 & -111.69 \\
Muennighoff & 7 & 0.02335 & 0.01713 & -136.29 & 0.03252 & 0.02711 & 0.03299 & -95.17 \\
SoftQ       & \textbf{5} & \textbf{0.00801} & \textbf{0.00520} & \textbf{-183.06}
             & \textbf{0.00595} & \textbf{0.00471} & \textbf{0.00785} & \textbf{-145.10} \\
\bottomrule
\end{tabular}
\end{table}

Table~\ref{tab:scaling-law-selection-main} shows that SoftQ is the strongest
baseline law across all three diagnostics.  It gives the best in-sample fit on
the full baseline dataset, the best data-axis extrapolation to the held-out $400$M
budget, and the best fit on the external scaling law datasets from \cite{kim2026pretraining}.  
The held-out result is especially important for data-efficiency estimation, because we later ask how
much additional unique data the regularized baseline would need to match an MIR
asymptote. Figure~\ref{fig:softq-gap-overview} visualizes this fit: SoftQ reproduces the empirical fan-out across data budgets that Chinchilla cannot.
Eq.~\eqref{eq:softq-fitted-main} in Appendix~\ref{app:scaling-law-details} gives its full expression.

We train the model with MIR on the same grid.
For each unique-token budget, we fit the parameter-scaling asymptote
$L_{\mathrm{MIR},U}(N)=E_U+A_U/N^{\alpha_U}$ using the five model sizes at that
budget. See Figure \ref{fig:app-dclm-scaling-8curves} for the fitted curves.
Taking $N\to\infty$ in Eq.~\eqref{eq:softq-fitted-main} gives the
regularized-baseline infinite-model curve
\(L_{\mathrm{Reg},\infty}(U)=
    0.306+2.249\,U^{-0.125}.\)
For each MIR asymptote $E_U$, we solve
$L_{\mathrm{Reg},\infty}(U_{\mathrm{eq}})=E_U$ and report
$U_{\mathrm{eq}}/U_{\mathrm{MIR}}$.
Under this baseline law, MIR consistently improves unique-data efficiency:
at $200$M--$400$M unique tokens, the regularized baseline would need about
$1.28$--$1.34\times$ as much unique data to match the MIR infinite-model asymptote. For completeness, we also use the other three scaling laws to calculate the data efficiency ratios. SoftQ gives the most conservative data efficiency ratio at \(U=\) 400M among all scaling laws. See Appendix \ref{sec:app-alt-efficiency-laws} for details.

\section{Related Work}

\paragraph{Classical Scaling Laws}
Empirical scaling laws have provided a central tool for predicting language-model loss 
as a function of model size, data, and compute. 
\cite{hestness2017deep, rosenfeld2020a} found that deep-learning generalization curves often follow power laws across model and dataset scales. 
For language modeling, \cite{kaplan2020scaling} showed that cross-entropy loss scales predictably 
with parameter count, dataset size, and training compute. 
\cite{henighan2020scaling} extended similar power-law behavior to other autoregressive generative domains. 
\cite{hoffmann2022training} revised the compute-optimal allocation problem and argued that model size 
and training tokens should be increased at comparable rates, 
leading to the Chinchilla recipe. 
These laws are highly effective in the abundant-data setting, 
but they typically treat processed tokens as fresh samples 
and therefore do not explicitly distinguish unique data from repeated epochs. 
This distinction becomes important once the available corpus size, 
rather than compute, becomes the binding resource.

\paragraph{Data-constrained Pretraining}
\cite{niklas2023} studied repeated-data training and proposed effective-resource scaling laws 
that account for the diminishing value of repeated tokens and excess parameters; 
they found that modest repetition can be close to fresh data, 
but that the marginal value of repetition eventually decays. 
\cite{kim2026pretraining} sharpened this into an infinite-compute, 
fixed-data viewpoint, showing that simply increasing epochs and parameters can overfit, 
and that much stronger regularization, especially substantially larger weight decay than 
standard practice, can improve the best attainable loss.
Recent work has also explored data-side and benchmark-driven approaches to this regime.
\citet{kim2026data} generate document-level synthetic rephrases and show that scaling these generations improves validation
loss on web text. The NanoGPT Slowrun benchmark
\citep{qlabs2026slowrun} similarly operationalizes fixed-data, high-compute language
modeling by fixing 100M FineWeb \citep{fineweb} tokens and ranking methods by validation loss with no compute limit.

\paragraph{Masked Diffusion Language Model}
Discrete and masked diffusion language models provide an alternative to left-to-right 
factorization by corrupting tokens and learning to reverse the corruption process. 
\cite{sahoo2024simple} proposed masked diffusion language models with effective training recipes.
\cite{nie2025large, bie2025llada2} scaled up the model and data size to train large-scale diffusion language models. 
In the data-constrained setting, \cite{prabhudesai2025diffusion} and \cite{ni2025diffusion} report that masked diffusion models 
can outperform autoregressive models under repeated-data training, 
attributing the gains to factors such as any-order prediction, dense denoising supervision, and implicit Monte Carlo augmentation.

\paragraph{Masking, noising, and denoising objectives}
Training on corrupted inputs has a long history as a regularization and representation-learning principle. 
In NLP, BERT popularized masked language modeling for bidirectional representation learning \citep{devlin2019bert}, 
while BART and T5 extended masking and denoising ideas to sequence-to-sequence pretraining 
through masked-span reconstruction, arbitrary text corruption, 
and span corruption \citep{lewis-etal-2020-bart, raffel2020exploring}. 
These objectives use masking as the main pretraining task and often change the architecture 
or inference interface relative to decoder-only autoregressive language modeling. 
MIR instead keeps the standard causal next-token objective and autoregressive decoding, 
using masking only as an auxiliary input perturbation during training.
Several works are closer to MIR in spirit because they inject masking or dropout 
into autoregressive or limited-data training. 
\cite{zhang-etal-2020-token} used word and token dropout as data augmentation and regularization 
in sequence modeling. 
\cite{pmlr-v267-zhuang25b} proposed Mask-Enhanced Autoregressive Prediction, 
which masks a small fraction of input tokens 
and then performs standard next-token prediction to improve retrieval and 
long-context behavior. \cite{wang2025entropy} masked low-entropy tokens to 
regularize multi-epoch training on limited domain data.

Closely related is \cite{gao2025makes}, which observed that random input masking drives data efficiency in dLLMs. Their study focused on how weight decay and input masking separately impact data efficiency. Tuning a 0.6B model, they showed that each technique independently reduces validation loss and narrows the dLLM-AR gap, finding that stronger regularization (0.5 weight decay, 0.1 MLP dropout) or dLLM-style inputs separately make AR models competitive with dLLMs. In contrast,
our evaluation of MIR focuses on the interaction between input masking and weight decay. Instead of treating them as independent alternatives, we demonstrate that dLLM-style masking provides a complementary benefit to weight decay and dropout when deployed as an auxiliary loss, further lowering AR validation loss. Moreover, this gain widens as models scale. To quantify this gain, we further developed a new coupled scaling law (SoftQ) specifically for this data-constrained regime, confirming that MIR yields superior scaling behavior and quantifiable data-efficiency gains.

\section{Discussion}

We study data-constrained, compute-rich pretraining along two axes, 
regularization and scaling. First, large weight decay substantially reduces dLLM validation loss; 
MIR, an auxiliary next-token loss on randomly masked inputs, 
further improves AR model validation loss on top of large weight decay 
across 72M to 1.4B parameters. Second, the additive Chinchilla law is misspecified in this regime 
because it decouples model and data size; we propose the SoftQ scaling law, 
which couples them and fits both our experiments and 
an independent grid from prior work better than existing alternatives.
Our study has several limitations. 
Experiments span up to 1.4B parameters and 400M unique tokens, small relative to frontier-scale pretraining. 
We held model architecture and optimizer fixed; varying these could yield further gains. Our protocol also relies on heavy per-cell hyperparameter search; a hyperparameter-transfer recipe for this regime is a natural next step.

\input{acknowledgements.tex}

\bibliographystyle{unsrtnat}
\bibliography{references}


\newpage
\appendix

\etocsettocdepth.toc{subsection}
\section*{Appendix}
\begingroup
\parindent=0em
\hypersetup{linkcolor=black}
\etocsetnexttocdepth{subsection}
\etocsettocstyle{\rule{\linewidth}{\tocrulewidth}\vskip0.5\baselineskip}{\rule{\linewidth}{\tocrulewidth}}
\tableofcontents
\endgroup

\section{Experiment Details}
\label{app:exp_details}

This appendix describes the compute setup, architecture ladder, data splits,
training recipes, hyperparameter searches, and auxiliary experimental results.
See the full data generation and training code in \href{https://github.com/yixinw-lab/dc_pretrain}{Github}. See the Wandb logs at \href{https://wandb.ai/zhiwei-xu2000/overtrain-dclm?nw=nwuserzhiweixu2000}{WandB}.

\subsection{Compute, Architecture, and Scaling Ladder}
All experiments can run on eight 80GB SXM H100 GPUs. The longest AR model run
completes in under 24 hours. The longest dLLM model run completes in under
48 hours.

We use a Llama-style decoder-only transformer \citep{llama3} with QK norm and
interleaved global-local self-attention as the model architecture. Compared to
the architecture used in \cite{kim2026pretraining}, we additionally use QK norm
and interleaved local and global attention. QK norm is widely used
\citep{olmo20242, gemma3t, olmo2025olmo, yang2025qwen3} in recent open-source
large language models to stabilize pretraining, and interleaved local and global
attention is also widely used \citep{gemma3t, olmo2025olmo, huang2026step} to
reduce compute and reduce KV cache size. We use the GPT-2 \citep{radford2019language}
tokenizer with one extra [MASK] token for random masking. The vocabulary size is
\(50258\).

We follow the scaling ladder
\[
\text{ScalingLadder}(k) = (kW_1, kL_1, S_1, B_1),
\]
where \(W_1=1024\) is the embedding dimension when \(k=1\), \(L_1=12\) is the number of layers when \(k=1\), \(S_1 = 2048\) is the sequence length, \(B_1 = 128\) is the total batch size, and  $k \in \{0.5, 0.75, 1, 1.5, 2\}$. Across the scaling ladder, the attention
head dimension is fixed at 64, while the depth, embedding dimension, MLP
dimension, and number of attention heads increase with scale. The resulting
models span from 71,965,952 parameters to 1,439,273,984 parameters.
Table~\ref{tab:scaling_ladder} summarizes the full architecture ladder.

\begin{table*}[t]
\caption{Scaling Ladder Details.}
\label{tab:scaling_ladder}
\centering
\small
\begin{tabular}{c c c c c c c}
\hline
$k$ & Layers & Embed dim & MLP dim & Heads & Head dim & Model size \\
\hline
0.5  & 6  & 512  & 1536 & 8  & 64 & 71,965,952 \\
0.75 & 9  & 768  & 2048 & 12 & 64 & 140,983,680 \\
1.0  & 12 & 1024 & 2816 & 16 & 64 & 257,190,400 \\
1.5  & 18 & 1536 & 4096 & 24 & 64 & 664,200,960 \\
2.0  & 24 & 2048 & 5632 & 32 & 64 & 1,439,273,984 \\
\hline
\end{tabular}
\end{table*}

\subsection{Data and Evaluation Splits}
Following \cite{kim2026pretraining}, we use DCLM-POOL \citep{dclm24}, an
open-source pretraining dataset containing \(240\)T tokens. We use the DCLM
subset generated by \cite{kim2026pretraining} to construct datasets with 100M,
200M, 300M, and 400M unique training tokens. Each smaller-budget dataset is a
subset of the corresponding larger-budget dataset. We always use the same
evaluation dataset, which contains 10M tokens from DCLM.

We also use Stack-V2 \citep{stackv224} to evaluate whether masked-input
regularization is beneficial for pretraining on code data. The corresponding
validation losses are reported in Table~\ref{tab:stack_v2_losses}.

\subsection{AR Training Recipe}
Unless stated otherwise, AR experiments use the AdamW optimizer \citep{adamw18}
with \(\beta_1=0.9\), \(\beta_2=0.95\), and \(\epsilon = 10^{-8}\). This config
is adopted from \cite{kim2026pretraining}. For all AR model pretraining, we use
the Warmup-stable-decay (WSD) \citep{wsd24} learning-rate schedule with \(1\%\)
of the total steps for linear warmup and \(10\%\) of the total steps for
warmdown. We set dropout rate to be \(0.1\).
In comparison, \cite{kim2026pretraining} uses cosine annealing. We
tried both schedules for standard AR model pretraining and found that WSD
performs better across all model sizes we tested.

\subsection{dLLM Baseline Protocol}
For dLLM pretraining, we adopt the config used in
\cite{prabhudesai2025diffusion}: batch size 256, sequence length 2048, learning
rate schedule with peak \(2\times 10^{-4}\), minimum \(2\times 10^{-5}\), 1\%
warmup, cosine decay, weight decay 0.1, and gradient clipping of 1.0. For the
number of epochs, we adopt the optimal values reported in
\cite{prabhudesai2025diffusion}: 500 epochs for the 257M and 664M models and
800 epochs for the 140M model. We calculate validation loss after each epoch and
report the lowest value. The 140M model achieves its lowest validation loss
\(3.646694\) at epoch 789, the 257M model achieves \(3.602763\) at epoch 483,
and the 664M model achieves \(3.680272\) at epoch 141. We set dropout rate to be \(0.1\).

Table~\ref{tab:dclm100m_dllm_comparison} reports the DCLM 100M validation
losses for the AR and dLLM recipes at the three model sizes where we run dLLM
pretraining. The strongly regularized dLLM uses the AR-tuned weight decay while
keeping the other dLLM hyperparameters fixed to the protocol above.

\begin{table}[H]
\caption{DCLM 100M validation loss for AR and dLLM recipes at different model sizes. For AR recipes, we report final validation loss; for dLLM
recipes, we report the best across-epoch validation loss.}
\label{tab:dclm100m_dllm_comparison}
\centering
\small
\begin{tabular}{lccc}
\toprule
& \multicolumn{3}{c}{Model size} \\
\cmidrule(lr){2-4}
Recipe & 140M & 257M & 664M \\
\midrule
Multi-Epoch dLLM & 3.646694 & 3.602763 & 3.680272 \\
Multi-Epoch AR & 3.945268 & 3.879782 & 3.821800 \\
Strongly Regularized dLLM & 3.579445 & 3.483598 & 3.387994 \\
Strongly Regularized AR & 3.471395 & 3.422107 & 3.367138 \\
MIR & \textbf{3.468458} & \textbf{3.404833} & \textbf{3.332668} \\
\bottomrule
\end{tabular}
\end{table}

\subsection{Multi-epoch AR Epoch Search}
We search for the best number of epochs for each model size for multi-epoch AR.
As shown in Figure \ref{fig:multiepochAR140_257_664}, 16 epochs is the best for
140M, 8 epochs is the best for 257M, and 32 epochs is the best for 664M.
\begin{figure}[H]
    \centering
    \includegraphics[width=0.32\linewidth]{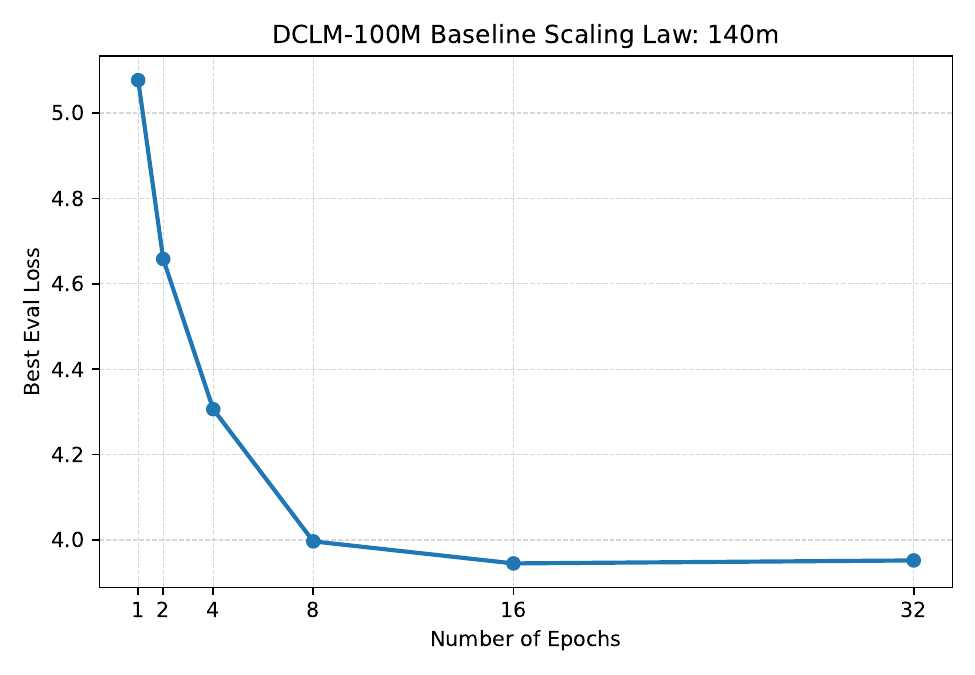}
    \includegraphics[width=0.32\linewidth]{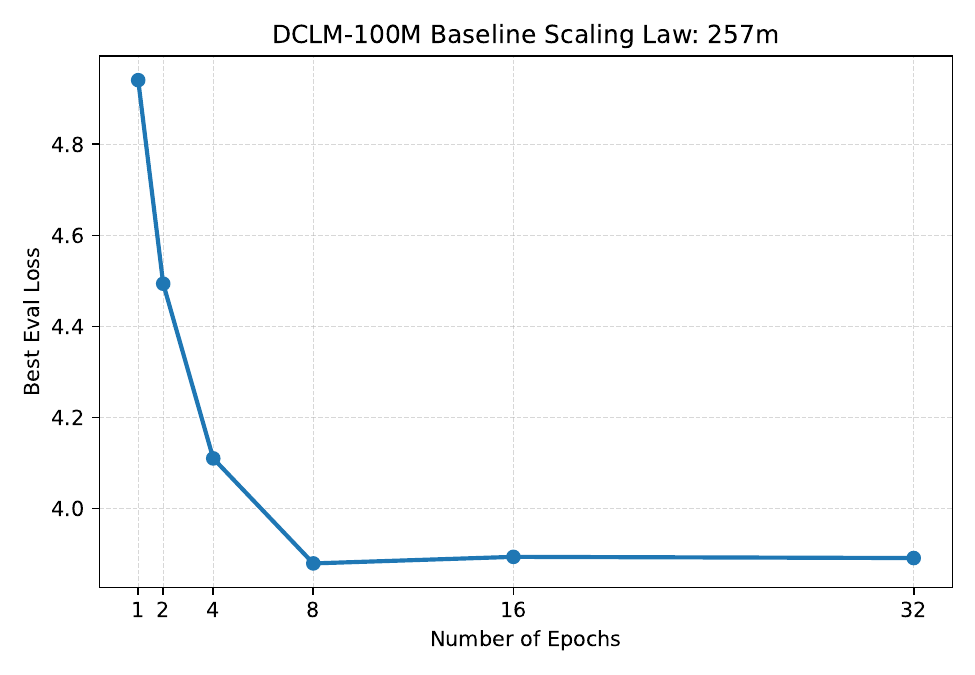}
    \includegraphics[width=0.32\linewidth]{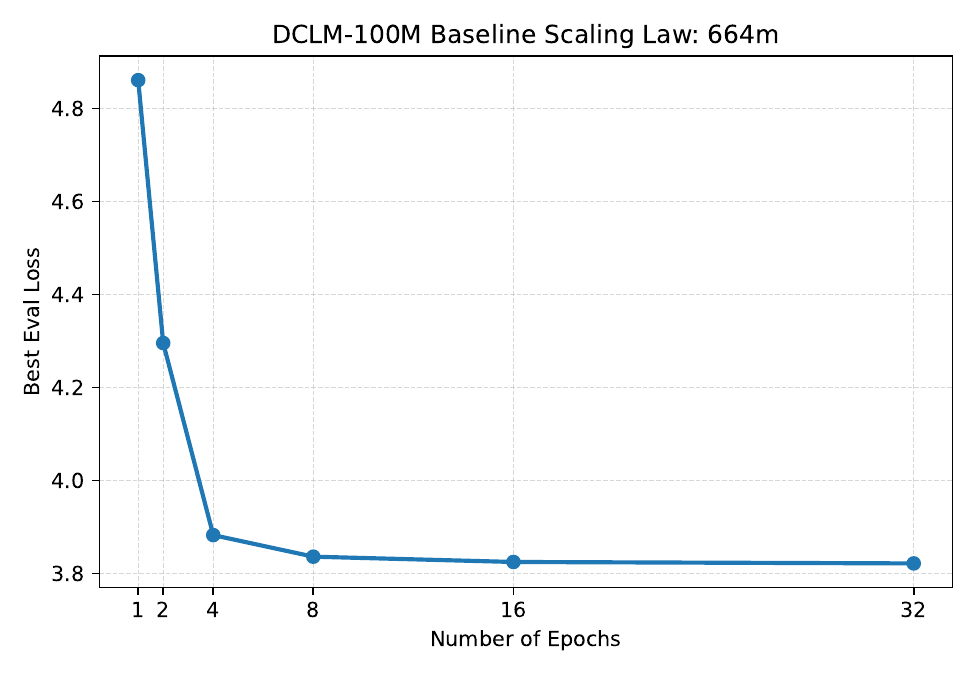}
    \caption{Validation loss vs. number of epochs. Weight decay is fixed to 0.1, 
    peak learning rate is fixed to 2e-4.
    Left: model size 140M; Middle: model size 257M; Right: model size 664M.}
    \label{fig:multiepochAR140_257_664}
\end{figure}

\subsection{Strongly Regularized Baseline Search}
The strongly regularized baseline sweeps are conducted in the data-constrained
DCLM setting described in the main text. We run separate searches at unique-data
budgets of 100M, 200M, 300M, and 400M tokens. Within each budget, we use the
same training and evaluation datasets across all model scales so that
differences in performance can be attributed to model size and training
objective rather than differences in data exposure.

We tune the optimization settings separately for each model scale and data
budget. The search space consists of the number of training epochs, weight
decay, and learning rate. In general, larger models prefer fewer epochs and
stronger weight decay, while the selected learning rates remain in the range of
$10^{-3}$ to $10^{-2}$. We describe the full 100M sweeps first, then append the
larger-budget searches used in the scaling-law analysis.

\textbf{72M model.}
We search over epochs $\{16, 32, 64\}$, weight decay
$\{0.4, 0.8, 1.6\}$, and learning rate
$\{10^{-3}, 3\times 10^{-3}, 10^{-2}\}$. We additionally run a refined sweep
over epochs $\{16, 32, 64\}$, weight decay $\{0.1, 0.2, 0.4\}$, and learning
rate $\{3\times 10^{-3}, 10^{-2}, 3\times 10^{-2}\}$. The best configuration
is $(32, 0.4, 10^{-2})$.

\textbf{140M model.}
We first search over epochs $\{8, 16, 32\}$, weight decay
$\{0.8, 1.6, 3.2\}$, and learning rate
$\{3\times 10^{-4}, 10^{-3}, 3\times 10^{-3}\}$. We then run an additional
sweep over epochs $\{16, 32, 64\}$, weight decay
$\{0.2, 0.4, 0.8, 1.6\}$, and learning rate
$\{10^{-3}, 3\times 10^{-3}, 10^{-2}, 3\times 10^{-2}\}$. The best
configuration is $(32, 0.8, 3\times 10^{-3})$.

\textbf{257M model.}
We search over epochs $\{8, 16, 32\}$, weight decay
$\{0.8, 1.6, 3.2\}$, and learning rate
$\{3\times 10^{-4}, 10^{-3}, 3\times 10^{-3}\}$. The best configuration is
$(16, 1.6, 10^{-3})$.

\textbf{664M model.}
We search over epochs $\{8, 16, 32\}$, weight decay
$\{0.8, 1.6, 3.2\}$, and learning rate
$\{3\times 10^{-4}, 10^{-3}, 3\times 10^{-3}\}$. The best configuration is
$(16, 1.6, 10^{-3})$.

\textbf{1.4B model.}
We search over epochs $\{4, 8, 16\}$, weight decay
$\{1.6, 3.2, 6.4\}$, and learning rate
$\{3\times 10^{-4}, 10^{-3}, 3\times 10^{-3}\}$. The best configuration is
$(16, 3.2, 10^{-3})$.

Table~\ref{tab:best_hparams} summarizes the selected hyperparameters at each
scale. 

\begin{table}[H]
\caption{Best strongly regularized hyperparameter configuration in the 100M unique-token setting.}
\label{tab:best_hparams}
\centering
\small
\begin{tabular}{l c}
\toprule
Model size & Best $(\text{epochs}, \text{weight decay}, \text{lr})$ \\
\midrule
72M  & $(32, 0.4, 10^{-2})$ \\
140M & $(32, 0.8, 3\times 10^{-3})$ \\
257M & $(16, 1.6, 10^{-3})$ \\
664M & $(16, 1.6, 10^{-3})$ \\
1.4B & $(16, 3.2, 10^{-3})$ \\
\bottomrule
\end{tabular}
\end{table}

Table~\ref{tab:best_hparams_larger_budgets} summarizes the selected
hyperparameters for the larger unique-data budgets used in the scaling-law
analysis. For 200M and 400M unique tokens, we run budget-specific sweeps. For
the intermediate 300M budget, we evaluate candidate configurations inherited
from the selected 200M and 400M settings at each model scale. The longest runs
take around 2 hours at 200M, 6 hours at 300M, and 8 hours at 400M on 8 H100s.

\begin{table}[H]
\caption{Selected strongly regularized hyperparameter configurations for the larger unique-data budgets. Each entry is
$(\text{epochs}, \text{weight decay}, \text{lr})$.}
\label{tab:best_hparams_larger_budgets}
\centering
\small
\begin{tabular}{lcc}
\toprule
Unique data & Model size & Best $(\text{epochs}, \text{weight decay}, \text{lr})$ \\
\midrule
200M & 72M & $(64, 0.2, 10^{-2})$ \\
200M & 140M & $(32, 0.4, 3\times10^{-3})$ \\
200M & 257M & $(16, 0.8, 10^{-3})$ \\
200M & 664M & $(16, 1.6, 10^{-3})$ \\
200M & 1.4B & $(16, 1.6, 10^{-3})$ \\
300M & 72M & $(64, 0.1, 10^{-2})$ \\
300M & 140M & $(64, 0.2, 3\times10^{-3})$ \\
300M & 257M & $(32, 0.8, 10^{-3})$ \\
300M & 664M & $(32, 0.8, 10^{-3})$ \\
300M & 1.4B & $(32, 1.6, 10^{-3})$ \\
400M & 72M & $(64, 0.1, 10^{-2})$ \\
400M & 140M & $(64, 0.2, 3\times10^{-3})$ \\
400M & 257M & $(32, 0.4, 10^{-3})$ \\
400M & 664M & $(32, 0.8, 10^{-3})$ \\
400M & 1.4B & $(32, 0.8, 10^{-3})$ \\
\bottomrule
\end{tabular}
\end{table}

\subsection{MIR Hyperparameter Tuning}\label{subsec:mir_coeff_tuning}
In MIR, for each sequence \(x\), a mask ratio \(r\) is sampled from \(\text{Unif}(r_{\min}, r_{\max})\), 
then for each position \(t \in [0, T-1]\), we use a Bernoulli random variable with success probability \(r\) 
to decide whether to mask \(x_t\). Denote the masked version as \(\tilde{x}\). We optimize
\[
\mathcal{L} = \mathcal{L}_{\mathrm{NTP}}(x) + \lambda \mathcal{L}_{\mathrm{NTP}}(\tilde{x}).
\]
We tune the values of \(r_{\min}, r_{\max}, \lambda\) using the \(1.4\)B model and DCLM 100M. 
See the results in Figure \ref{fig:reg_coeff_tuning}. The selected values are \(r_{\min}=0, r_{\max}=0.5, \lambda=0.4\). We also tried
\[
\mathcal{L} = (1-\lambda) \mathcal{L}_{\mathrm{NTP}}(x) + \lambda \mathcal{L}_{\mathrm{NTP}}(\tilde{x}),
\]
but its performance was slightly worse than \(\mathcal{L}_{\mathrm{NTP}}(x) + \lambda \mathcal{L}_{\mathrm{NTP}}(\tilde{x})\).
\begin{figure}[H]
    \centering
    \includegraphics[width=0.5\linewidth]{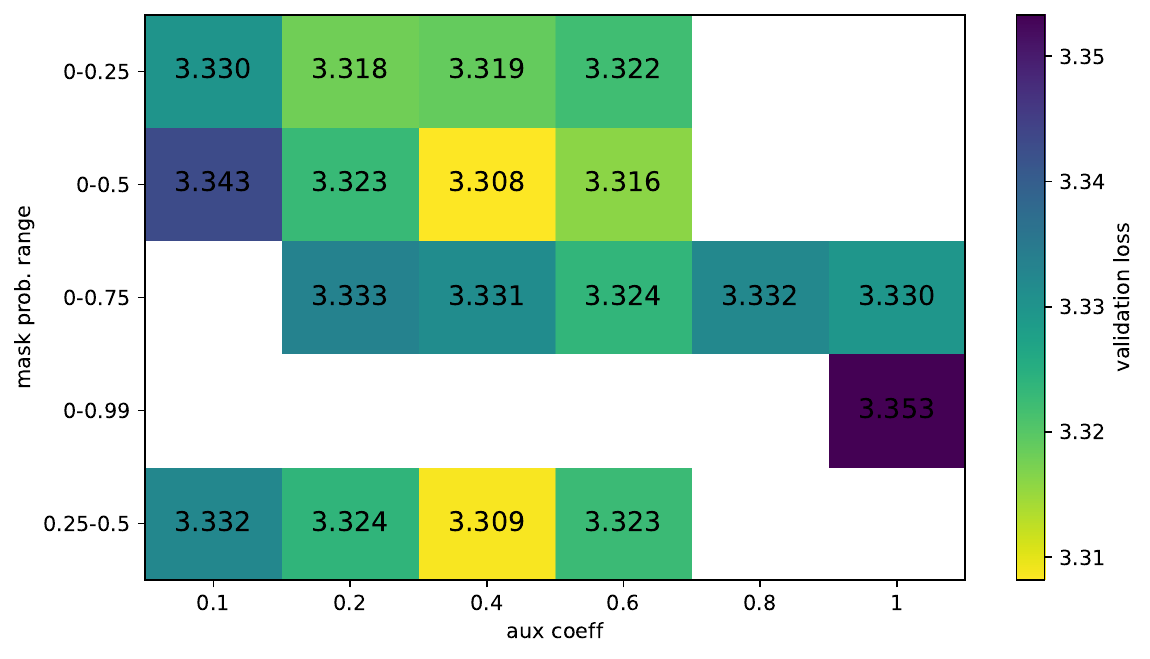}
    \caption{Tuning the mask ratio bounds (\(r_{\min}\), \(r_{\max}\)) and regularization coefficient \(\lambda\).}
    \label{fig:reg_coeff_tuning}
\end{figure}

\subsection{Auxiliary Experimental Results}
Table~\ref{tab:objective_comparison} reports the best evaluation loss across
model scales in the 100M unique-token setting. Table~\ref{tab:stack_v2_losses}
reports the Stack-V2 validation losses.

\begin{table}[H]
\caption{Best evaluation loss across model scales in the 100M unique-token setting (seed 42).}
\label{tab:objective_comparison}
\centering
\small
\begin{tabular}{l c c c c c}
\hline
Recipe & 72M & 140M & 257M & 664M & 1.4B \\
\hline
Single-epoch                           & 4.866105 & 4.960820 & 5.025738 & 5.019738 & 5.302995 \\
Strongly Regularized Recipe (Baseline) & 3.615903 & 3.471395 & 3.422107 & 3.367138 & 3.339578 \\
MIR                                    & 3.613621 & 3.468458 & 3.404833 & 3.332668 & 3.308170 \\
\hline
\end{tabular}
\end{table}

\begin{table}[H]
    \centering
    \caption{Validation loss on the Stack-V2 100M unique token dataset. MIR consistently outperforms the strongly regularized baseline across all model scales.}
    \label{tab:stack_v2_losses}
    \small
    \begin{tabular}{lrrrrr}
    \toprule
    Recipe & 72M & 140M & 257M & 664M & 1.4B \\
    \midrule
    Regularized Baseline & 1.064 & 1.020 & 1.005 & 0.996 & 0.983 \\
    MIR & \textbf{1.054} & \textbf{1.012} & \textbf{0.985} & \textbf{0.988} & \textbf{0.967} \\
    \bottomrule
    \end{tabular}
\end{table}

\subsection{Dataset Licenses}
\label{app:data_licenses}

\begin{table}[h]
\centering
\small
\begin{tabular}{p{0.1\linewidth} p{0.2\linewidth} p{0.1\linewidth} p{0.6\linewidth}}
\toprule
Dataset & Use in this paper & Version / URL & License and terms \\
\midrule
DCLM-Pool
& Natural-language pretraining and validation data.
& \href{https://huggingface.co/collections/mlfoundations/dclm-pools}{Link}
& CC BY 4.0. DCLM-Pool is derived from Common Crawl and is also subject to the Common Crawl Terms of Use. We cite the original DCLM paper and do not redistribute the raw dataset. \\

The Stack v2
& Code-heavy pretraining data for the Stack-v2 experiments.
& \href{https://huggingface.co/datasets/bigcode/the-stack-v2}{Link}. Version: 2.1.0

& No single dataset-wide content license; Hugging Face lists the license as ``other''. The dataset contains source code from repositories with various original licenses. User must comply with upstream licenses, including attribution clauses where relevant, the Stack-v2 access terms, Software Heritage principles for language-model training, and validated removal-request updates. We do not redistribute raw Stack-v2 files. \\
\bottomrule
\end{tabular}
\caption{Existing datasets used in this paper and their licenses or terms of use.}
\label{tab:data_licenses}
\end{table}

\section{Details of the Scaling-Law Analysis}
\label{app:scaling-law-details}

This appendix gives the full scaling-law definitions, fitting protocol, fitted
constants, residual diagnostics, and the plots moved out of the main text.

\subsection{Setup, Notation, and Fitting Objective}
We use $N$ for model size and $U$ for unique training tokens, both measured in
billions.  The baseline grid contains $5$ model sizes
$\{72\mathrm{M},140\mathrm{M},257\mathrm{M},664\mathrm{M},1.4\mathrm{B}\}$ and
$4$ unique-token budgets $\{100\mathrm{M},200\mathrm{M},300\mathrm{M},400\mathrm{M}\}$,
for $20$ strongly regularized baseline points.  The external grid
is provided in
\citet{kim2026pretraining}. For our dataset and the external one, the repetition variable in
the Muennighoff-style law is $R_D=\mathrm{epochs}-1$, where the epoch count is
taken from the best configuration or run identifier.

For Chinchilla, Quanta, and SoftQ, we minimize the Approach-3-style objective of
\citet{hoffmann2022training}:
\begin{equation}
    \min_{\theta}\sum_i
    \mathrm{Huber}_{10^{-3}}
    \left(
    \log \widehat L_\theta(N_i,U_i)-\log L_i
    \right).
    \label{eq:app-approach3}
\end{equation}
All reported RMSE, MAE, and RSS values are computed afterward on the raw
validation-loss scale.  AIC is the SSE-based Gaussian criterion
\begin{equation}
    \mathrm{AIC}=n\log(\mathrm{RSS}/n)+2k,
\end{equation}
where constants independent of the model are omitted.  Since the fitted objective
is Huber loss on log residuals, this AIC should be read as a common raw-loss
summary criterion rather than the exact likelihood optimized during fitting.

\subsection{Candidate Scaling Laws}
The Chinchilla law is
\begin{equation}
    L_{\mathrm{Ch}}(N,U)=E+\frac{A}{N^\alpha}+\frac{B}{U^\beta}.
    \label{eq:app-chinchilla}
\end{equation}
Its additive form implies a data-independent model-size gap:
$L_{\mathrm{Ch}}(N_1,U)-L_{\mathrm{Ch}}(N_2,U)$ does not depend on $U$.
Figure~\ref{fig:app-dclm-baseline-vs-data-size} shows that the observed DCLM
baseline gaps vary with the unique-token budget.

The Quanta-motivated joint law is
\begin{equation}
    L_{\mathrm{Q}}(N,U)
    =
    E+
    \left(
        \frac{A}{N}
        +
        \frac{B}{U^{1/(1+\alpha)}}
    \right)^{\alpha}.
    \label{eq:app-quanta}
\end{equation}
It couples the parameter and data axes before applying the outer power, so the
marginal value of additional parameters depends on the available data.

The Muennighoff-style law replaces raw resources by effective resources:
\begin{align}
    L_{\mathrm{M}}(N,U,R_D)
    &=
    E+\frac{A}{(N')^\alpha}+\frac{B}{(D')^\beta}, \\
    D'
    &=
    U+U R_D^\star\left(1-\exp[-R_D/R_D^\star]\right), \\
    N'
    &=
    U_N+U_N R_N^\star\left(1-\exp[-R_N/R_N^\star]\right).
\end{align}
Given the base Chinchilla coefficients, we compute the one-epoch optimal
parameter count
\begin{equation}
    N_{\mathrm{opt}}(U)
    =
    \left(\frac{\alpha A}{\beta B}\right)^{1/\alpha} U^{\beta/\alpha},
\end{equation}
then set $U_N=\min\{N,N_{\mathrm{opt}}(U)\}$ and
$R_N=N/U_N-1$. It has seven parameters to fit: $\{A,B,E,\alpha,\beta,R_N^\star,R_D^\star\}$.
Our main comparison uses a dataset-adapted two-stage protocol:
fit the base Chinchilla coefficients on the relevant split, then fit only
$R_N^\star$ and $R_D^\star$.  This is the appropriate comparison if the goal is
to evaluate the effective-resource functional form on our loss scale.  The
literal fixed-C4 coefficients from \citet{niklas2023} are included as an
ablation in Table~\ref{tab:app-muennighoff-fixed-c4}.

SoftQ is
\begin{equation}
    L_{\mathrm{SoftQ}}(N,U)
    =
    E+
    \left(
        A N^{-\rho}
        +
        B U^{-\rho/(1+\alpha)}
    \right)^{\alpha/\rho}.
    \label{eq:app-softq}
\end{equation}
When $\rho=1$, this reduces to the Quanta law.
The parameter $\rho$ controls the softness of the transition between the
parameter-limited and data-limited regimes.

\subsection{Fit Quality and Model Selection}

\begin{table}[H]
\centering
\caption{Full fit on all $20$ strongly regularized baseline
points.  Lower is better.}
\label{tab:app-dclm-full-fit}
\small
\begin{tabular}{lrrrr}
\toprule
Law & \# params & RMSE & MAE & AIC \\
\midrule
Chinchilla  & 5 & 0.026528 & 0.018016 & -135.18 \\
Quanta      & 4 & 0.012517 & 0.008889 & -167.23 \\
Muennighoff & 7 & 0.023345 & 0.017130 & -136.29 \\
SoftQ       & 5 & \textbf{0.008015} & \textbf{0.005204} & \textbf{-183.06} \\
\bottomrule
\end{tabular}
\end{table}

\begin{table}[H]
\centering
\caption{Fit on the DCLM $100$M/$200$M/$300$M baseline points and evaluation on
the held-out $400$M points.}
\label{tab:app-dclm-heldout}
\small
\begin{tabular}{lrrrr}
\toprule
Law & Train RMSE & Train MAE & Held-out RMSE & Held-out MAE \\
\midrule
Chinchilla  & 0.024636 & 0.016223 & 0.031063 & 0.025396 \\
Quanta      & 0.012430 & 0.008853 & 0.014975 & 0.012073 \\
Muennighoff & 0.023208 & 0.016216 & 0.032519 & 0.027111 \\
SoftQ       & \textbf{0.008850} & \textbf{0.005502} & \textbf{0.005955} & \textbf{0.004708} \\
\bottomrule
\end{tabular}
\end{table}

\begin{table}[H]
\centering
\caption{Held-out residuals on DCLM $400$M, predicted minus observed.}
\label{tab:app-dclm-heldout-residuals}
\small
\begin{tabular}{lrrrrr}
\toprule
Law & 72M & 140M & 257M & 664M & 1.4B \\
\midrule
Chinchilla  & -0.060050 & -0.024451 & -0.011210 & +0.017412 & +0.013857 \\
Quanta      & +0.028961 & +0.011817 & -0.009568 & -0.004273 & -0.005744 \\
Muennighoff & -0.061890 & -0.025911 & -0.011861 & +0.018631 & +0.017262 \\
SoftQ       & -0.002392 & +0.001378 & -0.008994 & -0.001468 & -0.009308 \\
\bottomrule
\end{tabular}
\end{table}

SoftQ has the best aggregate held-out RMSE and MAE.  It is closest to zero on
four of the five held-out model sizes; Quanta is slightly closer at the 1.4B
point.

\begin{table}[H]
\centering
\caption{Full fit on the regularized-baseline points provided by
\citet{kim2026pretraining}.}
\label{tab:app-paper16-full-fit}
\small
\begin{tabular}{lrrrr}
\toprule
Law & \# params & RMSE & MAE & AIC \\
\midrule
Chinchilla  & 5 & 0.040412 & 0.025554 & -92.68 \\
Quanta      & 4 & 0.023750 & 0.014726 & -111.69 \\
Muennighoff & 7 & 0.032989 & 0.022119 & -95.17 \\
SoftQ       & 5 & \textbf{0.007854} & \textbf{0.005955} & \textbf{-145.10} \\
\bottomrule
\end{tabular}
\end{table}

\subsection{Fitted Constants and Selected SoftQ Law}\label{subsec:fittedlaws}
See Table \ref{tab:app-dclm-params}, \ref{tab:app-heldout-params},
and \ref{tab:app-paper16-params} 
for the fitted constants of each scaling law in each scenario.
Specifically, on the full DCLM grid, the fitted SoftQ law is
\begin{equation}
\boxed{
L_{\mathrm{SoftQ}}(N,U)
=
0.30565+
\left(
39.2962\,N^{-0.79608}
+
92.4362\,U^{-0.69676}
\right)^{0.17906}
}
\label{eq:softq-fitted-main}
\end{equation}
with $N$ and $U$ in billions.  We therefore use
Eq.~\eqref{eq:softq-fitted-main} as the regularized-baseline law for the MIR
data-efficiency calculation.
\begin{table}[H]
\centering
\caption{Fitted constants on all $20$ DCLM baseline points.  For Muennighoff,
$A,\alpha,B,\beta,E$ are the first-stage Chinchilla coefficients.}
\label{tab:app-dclm-params}
\small
\setlength{\tabcolsep}{3.5pt}
\begin{tabular}{lrrrrrrl}
\toprule
Law & $A$ & $\alpha$ & $B$ & $\beta$ & $\rho$ & $E$ & Extra \\
\midrule
Chinchilla  & 0.1294 & 0.5167 & 0.5357 & 0.2924 & --     & 2.1116 & -- \\
Quanta      & 242.5882 & 0.1354 & 564.4767 & -- & -- & 0.2283 & -- \\
Muennighoff & 0.1294 & 0.5167 & 0.5357 & 0.2924 & --     & 2.1116 & $R_N^\star=31.39$, $R_D^\star=0.024$ \\
SoftQ       & 39.2962 & 0.1425 & 92.4362 & -- & 0.7961 & 0.3056 & -- \\
\bottomrule
\end{tabular}
\end{table}

\begin{table}[H]
\centering
\caption{Fitted constants for the held-out extrapolation experiment, trained only
on the DCLM $100$M/$200$M/$300$M points.}
\label{tab:app-heldout-params}
\small
\setlength{\tabcolsep}{3.0pt}
\begin{tabular}{lrrrrrrl}
\toprule
Law & $A$ & $\alpha$ & $B$ & $\beta$ & $\rho$ & $E$ & Extra \\
\midrule
Chinchilla  & 0.1363 & 0.4788 & 0.9823 & 0.1926 & --     & 1.6310 & -- \\
Quanta      & 799.5772 & 0.1263 & 1769.1342 & -- & -- & $5.4{\times}10^{-8}$ & -- \\
Muennighoff & 0.1363 & 0.4788 & 0.9823 & 0.1926 & --     & 1.6310 & $R_N^\star=90.51$, $R_D^\star=0.008$ \\
SoftQ       & 128.7280 & 0.1287 & 295.7854 & -- & 0.7853 & $1.9{\times}10^{-6}$ & -- \\
\bottomrule
\end{tabular}
\end{table}

\begin{table}[H]
\centering
\caption{Fitted constants on the external grid provided by \cite{kim2026pretraining}.}
\label{tab:app-paper16-params}
\small
\setlength{\tabcolsep}{3.5pt}
\begin{tabular}{lrrrrrrl}
\toprule
Law & $A$ & $\alpha$ & $B$ & $\beta$ & $\rho$ & $E$ & Extra \\
\midrule
Chinchilla  & 0.0543 & 1.1551 & 0.3659 & 0.4594 & --     & 2.6590 & -- \\
Quanta      & 0.1342 & 0.4959 & 0.4205 & -- & --     & 2.3197 & -- \\
Muennighoff & 0.0543 & 1.1551 & 0.3659 & 0.4594 & --     & 2.6590 & $R_N^\star=2.13$, $R_D^\star=0.096$ \\
SoftQ       & 0.0613 & 0.5905 & 0.2565 & -- & 1.4468 & 2.4360 & -- \\
\bottomrule
\end{tabular}
\end{table}

\subsection{MIR Data-Efficiency Calculation}

\begin{table}[H]
\centering
\caption{MIR parameter-scaling fits used for data-efficiency estimation.}
\label{tab:app-mir-param-fits}
\small
\begin{tabular}{rrrr}
\toprule
MIR unique data & $A_U$ & $\alpha_U$ & MIR asymptote $E_U$ \\
\midrule
100M & 0.03829 & 0.82186 & 3.27997 \\
200M & 0.13293 & 0.49592 & 2.95596 \\
300M & 0.13939 & 0.51307 & 2.83953 \\
400M & 0.15617 & 0.51006 & 2.74826 \\
\bottomrule
\end{tabular}
\end{table}

Using the full-DCLM SoftQ fit, the baseline infinite-model curve is
\[
L_{\mathrm{Reg},\infty}(U)=0.30565+2.24905\,U^{-0.12476},
\]
where $U$ is in billions.  Solving
$L_{\mathrm{Reg},\infty}(U_{\mathrm{eq}})=E_U$ gives the data-efficiency ratios
reported in Table~\ref{tab:mir-efficiency-softq-main}.

\begin{table}[H]
\centering
\caption{MIR unique-data efficiency relative to the strongly regularized baseline,
using SoftQ to model the baseline infinite-model curve.}
\label{tab:mir-efficiency-softq-main}
\small
\begin{tabular}{rrrr}
\toprule
MIR unique data & MIR asymptote $E_U$ & Baseline-equivalent $U_{\mathrm{eq}}$ & Data efficiency \\
\midrule
100M & 3.27997 & 106.4M & 1.06$\times$ \\
200M & 2.95596 & 268.2M & 1.34$\times$ \\
300M & 2.83953 & 384.5M & 1.28$\times$ \\
400M & 2.74826 & 515.9M & 1.29$\times$ \\
\bottomrule
\end{tabular}
\end{table}

\subsection{Sensitivity Analyses}
\label{sec:app-alt-efficiency-laws}

The original Muennighoff paper fixes the base Chinchilla coefficients to a
C4-calibrated law and fits only $R_N^\star,R_D^\star$.  Because those base
coefficients are on a different corpus and loss scale, they are not the main
comparison in this paper.  Table~\ref{tab:app-muennighoff-fixed-c4} reports the
literal fixed-C4 variant for completeness.

\begin{table}[H]
\centering
\caption{Literal fixed-C4 Muennighoff variant.  This fixes the base Chinchilla
law to the coefficients from \citet{niklas2023} and fits only
$R_N^\star,R_D^\star$.}
\label{tab:app-muennighoff-fixed-c4}
\small
\begin{tabular}{lrrrr}
\toprule
Dataset / split & RMSE & MAE & AIC & Notes \\
\midrule
DCLM full fit & 0.060257 & 0.047360 & -108.37 & $R_N^\star=119.82$, $R_D^\star=9.995$ \\
DCLM held-out $400$M & 0.071734 & 0.064728 & -- & train RMSE $=0.055268$ \\
Kim et al full fit & 0.120231 & 0.098342 & -63.79 & $R_N^\star=10^8$, $R_D^\star=0.927$ \\
\bottomrule
\end{tabular}
\end{table}

\begin{table}[H]
\centering
\caption{MIR data efficiency under each fitted regularized-baseline law.  We
fit each baseline law on the same 20 DCLM regularized points.  $U_{\mathrm{eq}}$
is the amount of unique data the corresponding regularized-baseline
infinite-model curve needs to match the MIR asymptote $E_U$.}
\label{tab:app-mir-efficiency-all-laws}
\small
\setlength{\tabcolsep}{3.2pt}
\begin{tabular}{lrrrrrrrr}
\toprule
& \multicolumn{2}{c}{Chinchilla} & \multicolumn{2}{c}{Quanta} & \multicolumn{2}{c}{Muennighoff} & \multicolumn{2}{c}{SoftQ} \\
\cmidrule(lr){2-3}\cmidrule(lr){4-5}\cmidrule(lr){6-7}\cmidrule(lr){8-9}
MIR $U$ & $U_{\mathrm{eq}}$ & Eff. & $U_{\mathrm{eq}}$ & Eff. & $U_{\mathrm{eq}}$ & Eff. & $U_{\mathrm{eq}}$ & Eff. \\
\midrule
100M & 69.5M  & 0.70$\times$ & 115.0M & 1.15$\times$ & 92.4M  & 0.92$\times$ & 106.4M & 1.06$\times$ \\
200M & 211.1M & 1.06$\times$ & 294.7M & 1.47$\times$ & 280.6M & 1.40$\times$ & 268.2M & 1.34$\times$ \\
300M & 350.6M & 1.17$\times$ & 424.9M & 1.42$\times$ & 466.1M & 1.55$\times$ & 384.5M & 1.28$\times$ \\
400M & 554.3M & 1.39$\times$ & 572.7M & 1.43$\times$ & 737.0M & 1.84$\times$ & 515.9M & 1.29$\times$ \\
\bottomrule
\end{tabular}
\end{table}

The main paper reports MIR data efficiency using SoftQ because it is the selected
baseline law by full-fit AIC, held-out prediction, and the external
check using data from \cite{kim2026pretraining}.  For completeness, we also compute the same quantity under the
Chinchilla, Quanta, and Muennighoff-style laws fitted on the full DCLM
strongly regularized baseline grid.  The Chinchilla, Quanta, and SoftQ fits use
the Approach-3 log-Huber objective in Eq.~\eqref{eq:app-approach3}.  The
Muennighoff-style fit uses the two-stage protocol described above: fit the base
Chinchilla coefficients on the same DCLM grid, then hold those
coefficients fixed and fit only $R_N^\star$ and $R_D^\star$.

For each law, we define the regularized-baseline infinite-model curve
$L_{\mathrm{Reg},\infty}(U)$ and solve
$L_{\mathrm{Reg},\infty}(U_{\mathrm{eq}})=E_U$, where $E_U$ is the MIR
parameter-scaling asymptote in Table~\ref{tab:app-mir-param-fits}.  The data
efficiency ratio is $U_{\mathrm{eq}}/U_{\mathrm{MIR}}$.  For Chinchilla,
Quanta, and SoftQ, the infinite-model curves are obtained by taking
$N\to\infty$.  For the Muennighoff-style law, we take both $N\to\infty$ and the
saturated repeated-data limit $R_D\to\infty$, giving
\[
L_{\mathrm{M},\infty}(U)
= E+
\frac{A}{\{(1+R_N^\star)N_{\mathrm{opt}}(U)\}^{\alpha}}
+\frac{B}{\{(1+R_D^\star)U\}^{\beta}},
\]
where
\[
N_{\mathrm{opt}}(U)=\left(\frac{\alpha A}{\beta B}\right)^{1/\alpha}
U^{\beta/\alpha}.
\]
With the fitted constants in Table~\ref{tab:app-dclm-params}, the resulting
one-dimensional curves are
\begin{align*}
L_{\mathrm{Ch},\infty}(U) &= 2.11164+0.53575\,U^{-0.29241},\\
L_{\mathrm{Q},\infty}(U) &= 0.22834+2.35787\,U^{-0.11924},\\
L_{\mathrm{M},\infty}(U) &= 2.11164+0.58227\,U^{-0.29241},\\
L_{\mathrm{SoftQ},\infty}(U) &= 0.30565+2.24905\,U^{-0.12476},
\end{align*}
where $U$ is measured in billions of unique tokens.

The alternative-law estimates vary substantially.  In particular, the additive
Chinchilla fit gives a sub-unity ratio at 100M because its infinite-model curve
already predicts a loss below the MIR asymptote at 100M, which is another
symptom of the decoupled law being misspecified in this regime.  Quanta and
Muennighoff generally produce larger ratios than SoftQ at 200M--400M, while
SoftQ gives the most conservative estimates among the coupled laws that also
passed the held-out and external-data checks.  For this reason, we keep the
SoftQ-based ratios as the main-text estimate and report the other laws only as
sensitivity analyses. 
We also observe that the data efficiency ratio difference under different scaling laws generally shrinks as the unique data size increases.

\subsection{Additional Visualizations}
Figure~\ref{fig:app-softq-baseline-fit} gives the absolute-loss view of the
Chinchilla and SoftQ fits. Figures~\ref{fig:app-dclm-baseline-vs-data-size}
and~\ref{fig:app-dclm-scaling-8curves} provide additional views of the baseline
and MIR scaling results.

\begin{figure}[H]
    \centering
    \begin{subfigure}[t]{0.48\linewidth}
        \centering
        \includegraphics[width=\linewidth]{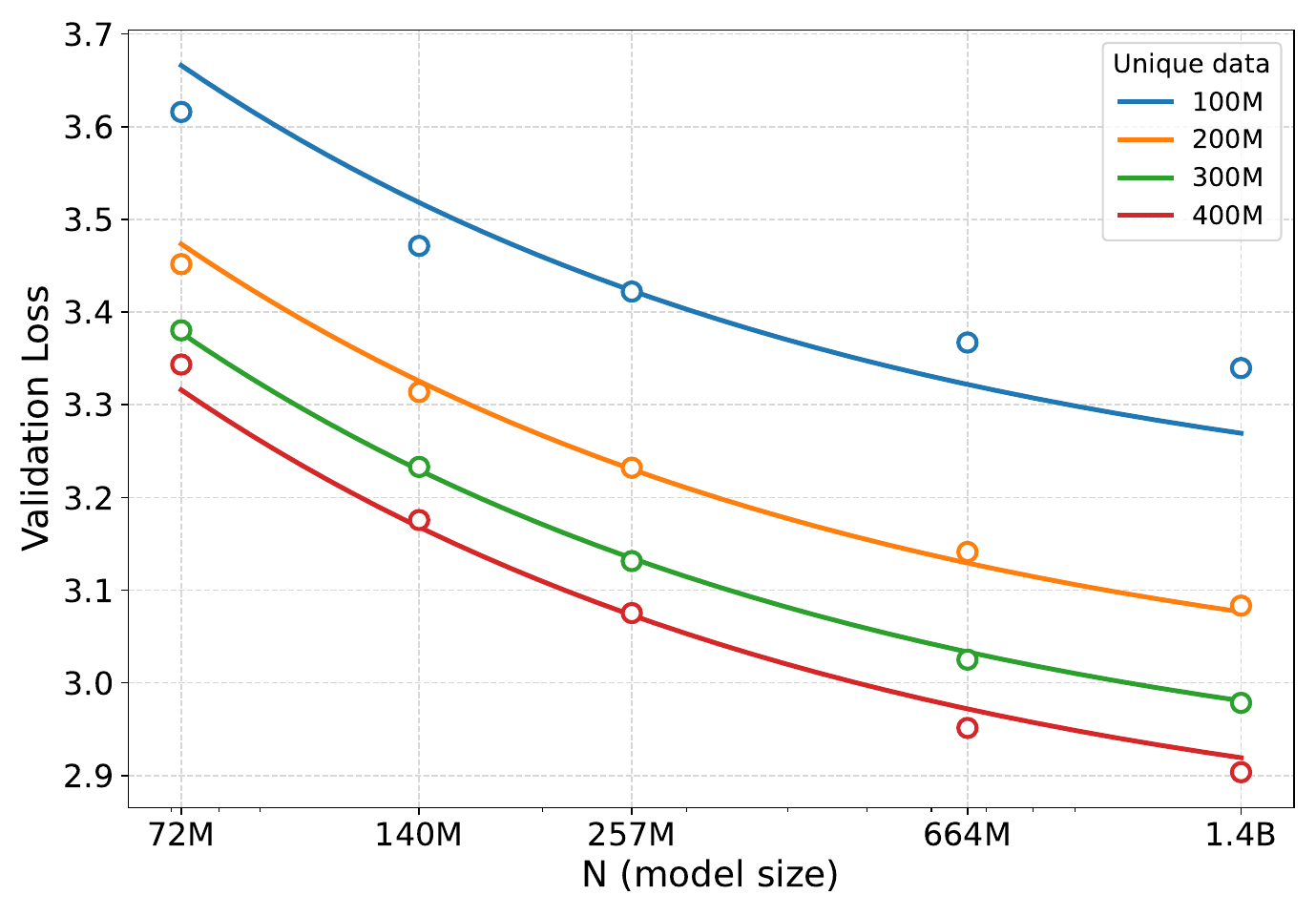}
    \end{subfigure}\hfill
    \begin{subfigure}[t]{0.48\linewidth}
        \centering
        \includegraphics[width=\linewidth]{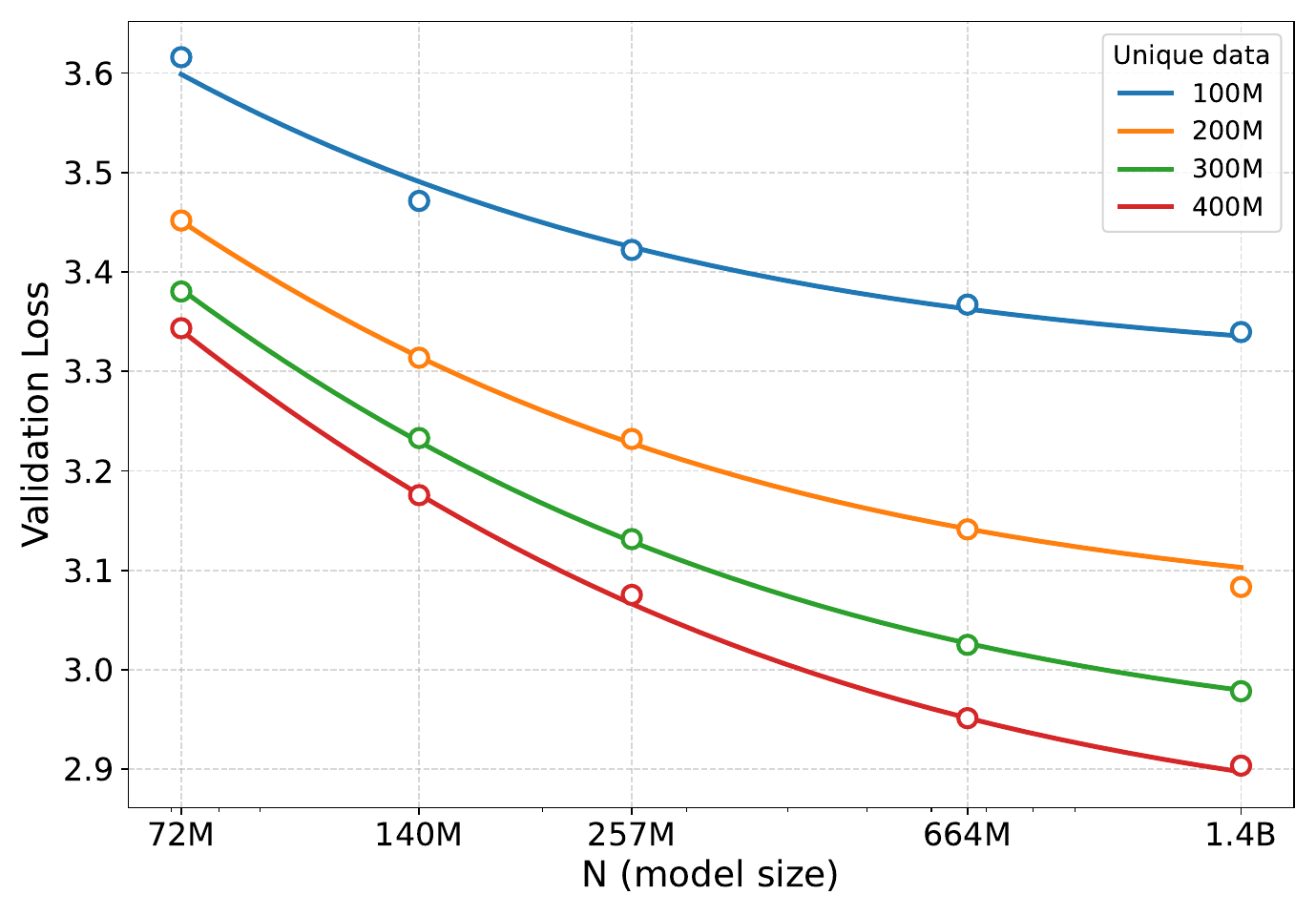}
    \end{subfigure}
    \caption{Absolute-loss view of the fitted Chinchilla and SoftQ laws on the
    \(20\) strongly regularized baseline points. Left: Chinchilla fit. Right:
    SoftQ fit. Points are observed validation losses and curves are model
    predictions.}
    \label{fig:app-softq-baseline-fit}
\end{figure}

\begin{figure}[H]
    \centering
    \includegraphics[width=0.62\linewidth]{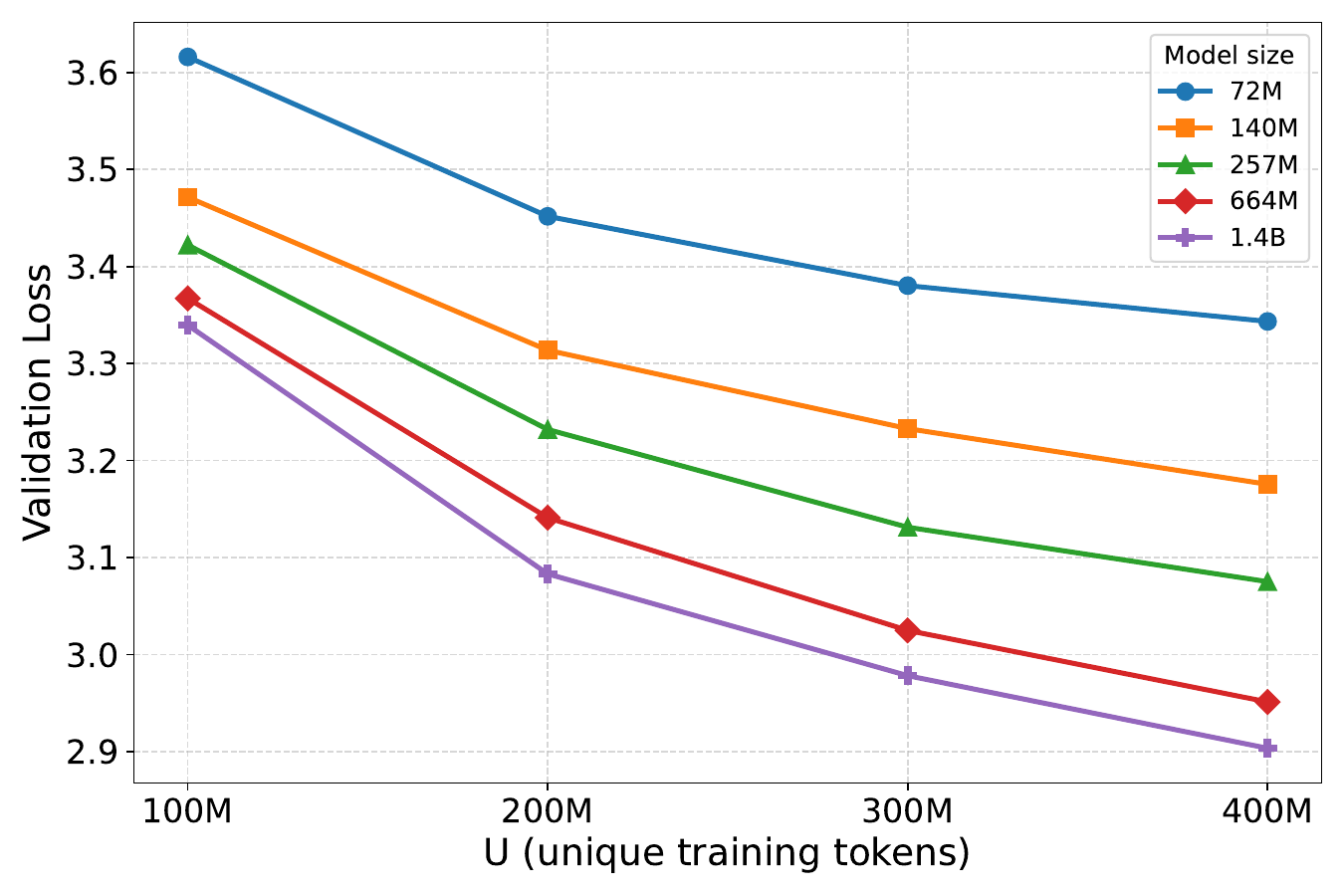}
    \caption{Regularized baseline validation loss as a function of unique training
    data size $U$ across five model sizes.  The changing separation between curves
    contradicts the data-independent model-size gap implied by the additive
    Chinchilla form.}
    \label{fig:app-dclm-baseline-vs-data-size}
\end{figure}

\begin{figure}[H]
    \centering
    \includegraphics[width=0.72\linewidth]{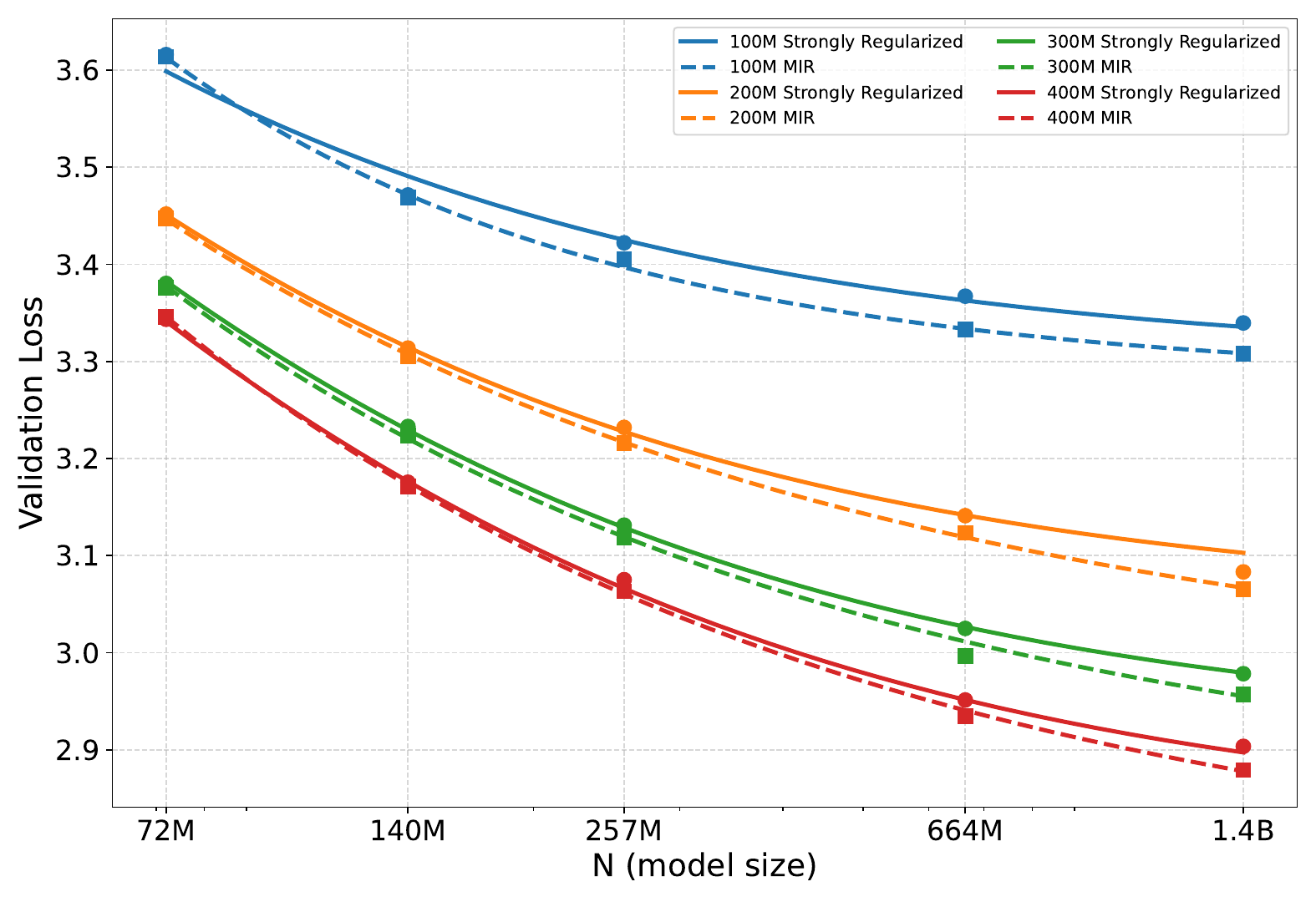}
    \caption{Scaling curves across four unique-data budgets for the strongly
    regularized baseline and MIR.  MIR improves validation loss at most model-size
    and data-budget pairs, and the asymptotic fits in Table~\ref{tab:app-mir-param-fits}
    quantify the infinite-model limit.}
    \label{fig:app-dclm-scaling-8curves}
\end{figure}

\clearpage

\section{Why Masking Reduces Memorization: A Toy Model}
\label{sec:theory}

This section gives a toy model for the intuition stated in Section \ref{sec:theo:intuition} of the main text:
masked-input regularization reduces validation loss by reducing dependence on
context-specific components (noise) and preserving a signal through generalizable
components. The intention in this section is not to model
transformer pretraining in full, but to isolate one mechanism that becomes
important in the data-constrained, compute-rich regime.

We decompose each training sequence into three parts: a context-specific
component, a generalizable component, and an output token. The context-specific
component can identify individual training examples and therefore enables
memorization. The generalizable component contains predictive information that
also appears in validation examples. A sufficiently large model can fit the
finite training set through the context-specific component alone; however, this
fit does not transfer to validation examples with unseen context-specific
components. Masking changes this because it can sometimes hide the context-specific component while leaving the generalizable component visible. On such masked inputs, prediction through memorization is unavailable, so the model is encouraged to learn predictive patterns from the generalizable component. Specifically, we introduce the following context-specific noise model.

\begin{definition}[Context-Specific Noise Model]
\label{def:context-specific-noise-model}
The training set consists of examples
\[
    (C_i,S_i,Y_i), ~~ i=1,\ldots,n,
\]
where \(C_i\) is a context-specific component, \(S_i\in\mathbb{R}^d\) is a
generalizable component, and \(Y_i\in\{-1,+1\}\) is the output token to be
predicted. In this model, we consider binary prediction for simplicity and clarity. We assume
\[
    \|S_i\|_2\le B.
\]
The population validation distribution has the same joint distribution of
\((S,Y)\), but its context-specific components are unseen during training.

Let
\[
    \mu := \mathbb{E}[YS]\in\mathbb{R}^d,
    ~~
    \Sigma := \mathbb{E}[SS^\top]\in\mathbb{R}^{d\times d},
\]
and assume \(\mu\neq 0\). On the finite training set, define
\[
    \widehat \mu := \frac1n\sum_{i=1}^n Y_iS_i,
    ~~
    \widehat \Sigma := \frac1n\sum_{i=1}^n S_iS_i^\top .
\]
Let \(\mathbf S\in\mathbb{R}^{n\times d}\) be the matrix with \(i\)-th row
\(S_i^\top\), and let \(Y=(Y_1,\ldots,Y_n)^\top\).
\end{definition}

\begin{definition}[Clean and MIR Objectives]
\label{def:clean-mir-objectives}
For model size \(m\), let
\[
    \phi_i=\phi_m(C_i)\in\mathbb{R}^m,
    ~~
    \Phi_m=(\phi_1,\ldots,\phi_n)^\top\in\mathbb{R}^{n\times m},
    ~~
    G_m=\Phi_m\Phi_m^\top .
\]
The model prediction score on example \(i\) is
\[
    \theta_i=\phi_i^\top w+b^\top S_i,
\]
where \(w\in\mathbb{R}^m\) models the context-specific memorization component
and \(b\in\mathbb{R}^d\) models the generalizable component.
We consider squared and logistic losses,
\[
    \ell_{\rm sq}(y,\theta)=\frac12(y-\theta)^2,
    ~~
    \ell_{\rm log}(y,\theta)=\log(1+\exp(-y\theta)).
\]
To model masking, let \(r\in[0,1]\) be a sampled mask ratio. Conditional on
\(r\), let \(V_{C,i},V_{S,i}\in\{0,1\}\) be independent visibility indicators
with
\[
    \mathbb{P}(V_{C,i}=1\mid r)=\mathbb{P}(V_{S,i}=1\mid r)=1-r .
\]
The masked prediction score on example \(i\) is then 
\[
    V_{C,i}\phi_i^\top w+V_{S,i}b^\top S_i .
\]
Thus masking may remove the context-specific component, the generalizable
component, both, or neither.
For loss \(\ell\in\{\ell_{\rm sq}, \ell_{\rm log}\}\), the clean objective is
\[
\widehat J_{\ell,{\rm clean}}^{(m)}(w,b)
=
\frac1n\sum_{i=1}^n
\ell\!\left(Y_i,\phi_i^\top w+b^\top S_i\right)
+
\frac{\rho_w}{2n}\|w\|_2^2
+
\frac{\rho_b}{2}\|b\|_2^2 ,
\]
and the MIR objective is
\[
\widehat J_{\ell,{\rm MIR}}^{(m)}(w,b)
=
\widehat J_{\ell,{\rm clean}}^{(m)}(w,b)
+
\frac{\lambda}{n}\sum_{i=1}^n
\mathbb{E}_M
\left[
\ell\!\left(
Y_i,
V_{C,i}\phi_i^\top w+V_{S,i}b^\top S_i
\right)
\right],
\]
where the expectation is over the masking randomness.
\end{definition}

\begin{assumption}[Growing Context-Specific Capacity]
\label{ass:growing-context-capacity}
For every \(m\), \(G_m\succ 0\). Let $
    a_m:=\lambda_{\min}(G_m)$,
then $
    a_m\to\infty
    ~\text{as}~
    m\to\infty $.
\end{assumption}

Assumption \ref{ass:growing-context-capacity} captures the data-constrained, compute-rich regime: the number
of training examples is fixed, while the capacity of the context-specific
memorization component grows. In this regime, for any fixed vector of prediction scores on the finite training set, the context-specific component can represent that vector with vanishing regularization cost as \(m\to\infty\).

\begin{theorem}[Behavior of the generalizable component in Clean and MIR training]
\label{thm:mir-generalization}
Let
\[
    h:=\mathbb{E}[(1-r)^2],
    ~~
    q:=\mathbb{E}[r(1-r)],
    ~~
    \beta:=\lambda q,
\]
and assume \(\beta>0\). Define
\[
    \alpha:=1+\lambda h,
    ~~
    \delta:=\alpha+\beta,
    ~~
    \eta:=\delta-\frac{\alpha^2}{\delta}
    =
    \frac{\beta(\delta+\alpha)}{\delta}.
\]
Under Assumption~\ref{ass:growing-context-capacity}, let
\(b_{{\rm clean},{\rm sq}}^{(m)}\), \(b_{{\rm MIR},{\rm sq}}^{(m)}\),
\(b_{{\rm clean},{\rm log}}^{(m)}\), and
\(b_{{\rm MIR},{\rm log}}^{(m)}\) denote the \(b\)-coordinates of minimizers of
the corresponding objectives. Then, as \(m\to\infty\),
\[
    b_{{\rm clean},{\rm sq}}^{(m)}\to 0,
    ~~
    b_{{\rm MIR},{\rm sq}}^{(m)}
    \to
    \bar b_{\rm sq}
    :=
    \beta
    \left(\rho_b I_d+\eta\widehat\Sigma\right)^{-1}
    \widehat\mu .
\]
For logistic loss,
\[
    b_{{\rm clean},{\rm log}}^{(m)}\to 0,
    ~~
    b_{{\rm MIR},{\rm log}}^{(m)}\to \bar b_{\rm log},
\]
where \(\bar b_{\rm log}\) is the unique minimizer of
\[
    b\mapsto
    \beta\,
    \frac1n\sum_{i=1}^n
    \log\!\left(1+\exp(-Y_i b^\top S_i)\right)
    +
    \frac{\rho_b}{2}\|b\|_2^2 .
\]
Moreover, if \(\widehat\mu\neq 0\), then $
    \widehat\mu^\top \bar b_{\rm sq}>0$ and $
    \bar b_{\rm sq}\neq 0 $.
For logistic loss, if \(\widehat\mu\neq 0\), then
\[
    \bar b_{\rm log}\neq 0,
    ~~
    \|\bar b_{\rm log}\|_2\le \frac{\beta B}{\rho_b}.
\]
\end{theorem}

The theorem formalizes the memorization effect. Clean training can fit the
finite training set through the context-specific component alone, so the
coefficient on the generalizable component vanishes as context-specific capacity
grows. MIR does not have this degeneracy: because masking sometimes hides the
context-specific component, the limiting objective retains a nonzero training
signal for the generalizable component.

\begin{assumption}[Validation Contexts Are Unseen]
\label{ass:unseen-contexts}
For validation examples, the context-specific memorization features learned on
the training set are unavailable. We model this as
\[
    \phi_m(C_{\rm val})=0 .
\]
Therefore validation predictions depend only on the generalizable logit
\(b^\top S\).
\end{assumption}

This assumption does not say that validation text contains no patterns related
to the training text. It says only that the example-specific context features
used to memorize the finite training corpus do not transfer to unseen validation
examples.

\begin{theorem}[MIR Improves Validation Risk]
\label{thm:mir-improves-validation-risk}
Under the assumptions of Theorem~\ref{thm:mir-generalization} and
Assumption~\ref{ass:unseen-contexts}, define
\[
    R_{\rm sq}(b)
    :=
    \mathbb{E}\!\left[(Y-b^\top S)^2\right]
    =
    1-2\mu^\top b+b^\top\Sigma b,
\]
and
\[
    R_{\rm log}(b)
    :=
    \mathbb{E}\!\left[
    \log\!\left(1+\exp(-Yb^\top S)\right)
    \right].
\]

For squared loss, if
$
    2\mu^\top \bar b_{\rm sq}
    -
    \bar b_{\rm sq}^\top\Sigma \bar b_{\rm sq}
    >
    0$,
then, for all sufficiently large \(m\),
\[
    R_{\rm sq}\!\left(b_{{\rm MIR},{\rm sq}}^{(m)}\right)
    <
    R_{\rm sq}\!\left(b_{{\rm clean},{\rm sq}}^{(m)}\right).
\]
This condition holds automatically when $
    \widehat\mu=\mu
    $ and $
    \widehat\Sigma=\Sigma $.
For logistic loss, if $
    \mu^\top \bar b_{\rm log}
    >
    \frac{B^2}{4}\|\bar b_{\rm log}\|_2^2$,
then, for all sufficiently large \(m\),
\[
    R_{\rm log}\!\left(b_{{\rm MIR},{\rm log}}^{(m)}\right)
    <
    R_{\rm log}\!\left(b_{{\rm clean},{\rm log}}^{(m)}\right).
\]
In particular, this logistic condition holds for sufficiently small
\(\beta/\rho_b\) whenever \(\mu^\top\widehat\mu>0\).
Moreover, defining
\[
    \Delta_{{\rm sq},m}
    :=
    R_{\rm sq}\!\left(b_{{\rm clean},{\rm sq}}^{(m)}\right)
    -
    R_{\rm sq}\!\left(b_{{\rm MIR},{\rm sq}}^{(m)}\right),
\]
and
\[
    \Delta_{{\rm log},m}
    :=
    R_{\rm log}\!\left(b_{{\rm clean},{\rm log}}^{(m)}\right)
    -
    R_{\rm log}\!\left(b_{{\rm MIR},{\rm log}}^{(m)}\right),
\]
we have
\[
    \Delta_{{\rm sq},m}
    \to
    R_{\rm sq}(0)-R_{\rm sq}(\bar b_{\rm sq})
    >
    0
\]
under the squared-loss condition, and
\[
    \Delta_{{\rm log},m}
    \to
    R_{\rm log}(0)-R_{\rm log}(\bar b_{\rm log})
    >
    0
\]
under the logistic-loss condition.
\end{theorem}

\begin{corollary}[Empirical Signal]
\label{cor:hp-empirical-signal}
Suppose \((S_i,Y_i)\) are independent, \(\|S_i\|_2\le B\), and
\(\mu=\mathbb{E}[YS]\neq 0\). Then
\[
    \mathbb{P}\!\left(\mu^\top\widehat\mu>0\right)
    \ge
    1-\exp\!\left(
        -\frac{n\|\mu\|_2^2}{2B^2}
    \right).
\]
In particular, with high probability, the empirical generalizable component is
aligned with the population predictive direction.
\end{corollary}

The previous results compare clean and MIR training in the limit.
The next result makes the dependence on model size explicit for a simplified
squared-loss objective. This simplified objective replaces the full expected
masked loss by the term that appears when the context-specific component is
hidden while the generalizable component remains visible. It is not the full MIR
objective, but it isolates the part of masking that forces prediction from the
generalizable component.

\begin{theorem}[Increasing Benefit with Growing Model Size]
\label{thm:compute-dependent-gains}
Consider the squared-loss objective
\[
    \widehat J_{{\rm key},{\rm sq}}^{(m)}(w,b)
    :=
    \widehat J_{\ell_{\rm sq},{\rm clean}}^{(m)}(w,b)
    +
    \frac{\beta}{2n}
    \|Y-\mathbf S b\|_2^2,
    \qquad
    \beta>0.
\]
Let \(b_{{\rm key},{\rm sq}}^{(m)}\) be its \(b\)-coordinate minimizer, and define
\[
    \Delta_{{\rm key},m}
    :=
    R_{\rm sq}\!\left(b_{{\rm clean},{\rm sq}}^{(m)}\right)
    -
    R_{\rm sq}\!\left(b_{{\rm key},{\rm sq}}^{(m)}\right).
\]
Assume $
    G_m=mI_n,
    \widehat\mu=\mu,
    \widehat\Sigma=\Sigma,
    \mu\neq 0 $.
Then \(\Delta_{{\rm key},m}\) is strictly increasing in \(m\).
Moreover, let
\[
    \Sigma=U{\rm diag}(\lambda_1,\ldots,\lambda_d)U^\top,
    ~~
    U^\top\mu=(\mu_1,\ldots,\mu_d)^\top .
\]
For each \(\lambda_j>0\), define
\[
    \kappa_j:=\frac{\beta\lambda_j}{\rho_b}.
\]
Then
\[
    \lim_{m\to\infty}\Delta_{{\rm key},m}
    =
    \sum_{\lambda_j>0}
    \frac{\mu_j^2}{\lambda_j}
    \frac{\kappa_j(\kappa_j+2)}{(1+\kappa_j)^2}
    >
    0 .
\]
\end{theorem}

Theorem~\ref{thm:compute-dependent-gains} illustrates the increasing benefit of
masking the context-specific component as model size increases.
The condition \(G_m=mI_n\) is an idealized isotropic-capacity assumption and is
stronger than needed for the main intuition; it is used here to obtain a simple
closed-form expression and a monotonicity statement. More general Gram matrices
with growing eigenvalues would lead to a similar conclusion, although the
closed-form expression would be less transparent.
The assumptions \(\widehat\mu=\mu\) and \(\widehat\Sigma=\Sigma\) remove
finite-sample realization error from the training sequences from the statement.
They ensure that the empirical generalizable signal in the training set is
aligned with the population signal that determines validation risk. Under these
conditions, any difference between clean training and the masked objective comes
from the use of context-specific memorization. In finite samples, these
assumptions can be interpreted as a population-aligned simplification: when
\(n\) is large, \(\widehat\mu\) and \(\widehat\Sigma\) concentrate around
\(\mu\) and \(\Sigma\), so the same conclusion is stable up to small perturbation
terms. In what follows, we prove the theoretical results in this section.

Throughout the proofs, we write
\[
    u:=\Phi_m w\in\mathbb{R}^n .
\]
Whenever \(G_m=\Phi_m\Phi_m^\top\succ0\), every \(u\in\mathbb{R}^n\) is
representable as \(\Phi_mw\), and the minimum-norm representative satisfies
\[
    \min_{w:\Phi_mw=u}\|w\|_2^2
    =
    u^\top G_m^{-1}u .
\]
Indeed, \(w_\star=\Phi_m^\top G_m^{-1}u\) satisfies
\(\Phi_mw_\star=u\). For any other feasible \(w=w_\star+v\), we have
\(\Phi_mv=0\), and hence
\[
    \langle w_\star,v\rangle
    =
    u^\top G_m^{-1}\Phi_m v
    =
    0.
\]
Thus
\[
    \|w\|_2^2
    =
    \|w_\star\|_2^2+\|v\|_2^2
    \ge
    \|w_\star\|_2^2
    =
    u^\top G_m^{-1}u .
\]
Therefore the optimization over \((w,b)\) is equivalent to optimization over
\((u,b)\), with regularization term
\[
    \frac{\rho_w}{2n}u^\top G_m^{-1}u .
\]

\subsection{Proof of Theorem~\ref{thm:mir-generalization}}

\begin{proof}
We first prove the squared-loss claims. In \((u,b)\)-coordinates, the clean
squared-loss objective is
\[
    \frac1{2n}\|Y-u-\mathbf S b\|_2^2
    +
    \frac{\rho_w}{2n}u^\top G_m^{-1}u
    +
    \frac{\rho_b}{2}\|b\|_2^2 .
\]
For fixed \(b\), the first-order condition in \(u\) is
\[
    u-(Y-\mathbf S b)+\rho_wG_m^{-1}u=0.
\]
Therefore
\[
    u^\star(b)
    =
    (G_m+\rho_w I_n)^{-1}G_m(Y-\mathbf S b).
\]
Profiling out \(u\), the clean squared-loss objective becomes
\[
    \frac1{2n}
    (Y-\mathbf S b)^\top
    T_m
    (Y-\mathbf S b)
    +
    \frac{\rho_b}{2}\|b\|_2^2,
    ~~
    T_m:=\rho_w(G_m+\rho_w I_n)^{-1}.
\]
Differentiating with respect to \(b\) gives
\[
    b_{{\rm clean},{\rm sq}}^{(m)}
    =
    \left(
        \rho_b I_d+\frac1n\mathbf S^\top T_m\mathbf S
    \right)^{-1}
    \frac1n\mathbf S^\top T_mY .
\]
The eigenvalues of \(T_m\) are
\[
    \frac{\rho_w}{\lambda_j(G_m)+\rho_w},
\]
and hence
\[
    \|T_m\|_{\rm op}
    \le
    \frac{\rho_w}{a_m+\rho_w}
    \to0 .
\]
It follows that
\[
    b_{{\rm clean},{\rm sq}}^{(m)}\to0 .
\]

We next consider MIR with squared loss. Let
\[
    s_0:=\mathbb{E}[r^2].
\]
Up to the additive constant \(\lambda s_0(2n)^{-1}\|Y\|_2^2\), the expected
four-case squared-loss objective is
\[
\begin{aligned}
    &
    \alpha\frac1{2n}\|Y-u-\mathbf S b\|_2^2
    +
    \beta\frac1{2n}\|Y-\mathbf S b\|_2^2
    +
    \beta\frac1{2n}\|Y-u\|_2^2
    \\
    &\hspace{3em}
    +
    \frac{\rho_w}{2n}u^\top G_m^{-1}u
    +
    \frac{\rho_b}{2}\|b\|_2^2 .
\end{aligned}
\]
The first-order condition in \(u\) is
\[
    \alpha(u+\mathbf S b-Y)+\beta(u-Y)+\rho_wG_m^{-1}u=0.
\]
Since \(\delta=\alpha+\beta\), this gives
\[
    u^\star(b)
    =
    M_m(\delta Y-\alpha\mathbf S b),
    ~~
    M_m:=(\delta I_n+\rho_wG_m^{-1})^{-1}.
\]
The first-order condition in \(b\) is
\[
    \frac{\alpha}{n}\mathbf S^\top(u+\mathbf S b-Y)
    +
    \frac{\beta}{n}\mathbf S^\top(\mathbf S b-Y)
    +
    \rho_b b
    =
    0.
\]
Equivalently,
\[
    \frac{\alpha}{n}\mathbf S^\top u
    +
    (\delta\widehat\Sigma+\rho_b I_d)b
    -
    \delta\widehat\mu
    =
    0 .
\]
Substituting \(u^\star(b)=M_m(\delta Y-\alpha\mathbf S b)\) yields
\[
    b_{{\rm MIR},{\rm sq}}^{(m)}
    =
    \left(
        \rho_b I_d
        +
        \delta\widehat\Sigma
        -
        \frac{\alpha^2}{n}\mathbf S^\top M_m\mathbf S
    \right)^{-1}
    \left(
        \delta\widehat\mu
        -
        \frac{\alpha\delta}{n}\mathbf S^\top M_mY
    \right).
\]
Because \(\|G_m^{-1}\|_{\rm op}\to0\), we have
\[
    M_m\to\delta^{-1}I_n
\]
in operator norm. Therefore
\[
    \frac1n\mathbf S^\top M_mY\to\frac1\delta\widehat\mu,
    ~~
    \frac1n\mathbf S^\top M_m\mathbf S\to\frac1\delta\widehat\Sigma .
\]
It follows that
\[
    b_{{\rm MIR},{\rm sq}}^{(m)}
    \to
    \beta(\rho_b I_d+\eta\widehat\Sigma)^{-1}\widehat\mu
    =
    \bar b_{\rm sq}.
\]
If \(\widehat\mu\neq0\), then
\[
    \widehat\mu^\top\bar b_{\rm sq}
    =
    \beta
    \widehat\mu^\top
    (\rho_b I_d+\eta\widehat\Sigma)^{-1}
    \widehat\mu
    >
    0,
\]
because \(\rho_b I_d+\eta\widehat\Sigma\succ0\). Hence
\(\bar b_{\rm sq}\neq0\).

We now prove the logistic-loss claims. Let
\[
    g(z):=\log(1+e^{-z}).
\]
Choose \(\tau_m\to\infty\) such that \(\tau_m^2/a_m\to0\). For clean logistic
training, evaluate the objective at \(u=\tau_mY\) and \(b=0\). Then
\[
    g(Y_iu_i)=g(\tau_m)\to0,
\]
and
\[
    \frac{\rho_w}{2n}u^\top G_m^{-1}u
    \le
    \frac{\rho_w}{2}\frac{\tau_m^2}{a_m}
    \to0 .
\]
Hence the minimum clean logistic objective converges to zero. Since the
objective is nonnegative and contains the term
\(\rho_b\|b\|_2^2/2\), every clean logistic minimizer satisfies
\[
    b_{{\rm clean},{\rm log}}^{(m)}\to0 .
\]

For MIR logistic training, define
\[
    F_{\rm log}(b)
    :=
    \beta
    \frac1n\sum_{i=1}^n
    g(Y_i b^\top S_i)
    +
    \frac{\rho_b}{2}\|b\|_2^2 .
\]
This function is strongly convex and therefore has a unique minimizer
\(\bar b_{\rm log}\). The expected four-case MIR logistic objective in
\((u,b)\)-coordinates is
\[
\begin{aligned}
     \widehat J_{\ell_{\rm log},{\rm MIR}}^{(m)}(u,b):= &~
    \alpha\frac1n\sum_{i=1}^n g(Y_i(u_i+b^\top S_i))
    +
    \beta\frac1n\sum_{i=1}^n g(Y_i b^\top S_i)
    +
    \beta\frac1n\sum_{i=1}^n g(Y_iu_i)
    \\
    &\hspace{3em}
    +
    \lambda s_0\log2
    +
    \frac{\rho_w}{2n}u^\top G_m^{-1}u
    +
    \frac{\rho_b}{2}\|b\|_2^2 .
\end{aligned}
\]
All terms except \(F_{\rm log}(b)\) and the constant \(\lambda s_0\log2\) are
nonnegative. Hence, for every \((u,b)\),
\[
    \widehat J_{\ell_{\rm log},{\rm MIR}}^{(m)}(u,b)
    \ge
    F_{\rm log}(b)+\lambda s_0\log2 .
\]
Now evaluate the MIR logistic objective at \(b=\bar b_{\rm log}\) and
\(u=\tau_mY\). Then the context-specific-only margin is
\(Y_i u_i=\tau_m\), while the margin satisfies
\[
    Y_i(u_i+\bar b_{\rm log}^\top S_i)
    =
    \tau_m+Y_i\bar b_{\rm log}^\top S_i
    \ge
    \tau_m-B\|\bar b_{\rm log}\|_2
    \to\infty .
\]
Thus the corresponding logistic losses vanish, and the context-specific
regularization again tends to zero as we choose $\tau_m$ such that \(\tau_m^2/a_m\to0\). Therefore
\[
    \inf_{u,b}\widehat J_{\ell_{\rm log},{\rm MIR}}^{(m)}(u,b)
    \le
    F_{\rm log}(\bar b_{\rm log})+\lambda s_0\log2+o(1).
\]
Combining the lower and upper bounds gives
\[
    F_{\rm log}\!\left(b_{{\rm MIR},{\rm log}}^{(m)}\right)
    \le
    F_{\rm log}(\bar b_{\rm log})+o(1).
\]
By strong convexity of \(F_{\rm log}\),
\[
    b_{{\rm MIR},{\rm log}}^{(m)}\to\bar b_{\rm log}.
\]

Finally,
\[
    \nabla F_{\rm log}(0)
    =
    -\frac{\beta}{2}\widehat\mu .
\]
Thus, if \(\widehat\mu\neq0\), zero is not the minimizer and
\(\bar b_{\rm log}\neq0\). At the minimizer \(\bar b_{\rm log}\), the first-order condition for
\(F_{\rm log}\) gives
\[
    0
    =
    \beta \frac1n\sum_{i=1}^n
    g'(Y_i \bar b_{\rm log}^\top S_i)Y_iS_i
    +
    \rho_b\bar b_{\rm log}.
\]
Equivalently,
\[
    \rho_b\bar b_{\rm log}
    =
    -\beta \frac1n\sum_{i=1}^n
    g'(Y_i \bar b_{\rm log}^\top S_i)Y_iS_i .
\]
Taking Euclidean norms and using the triangle inequality,
\[
\begin{aligned}
    \rho_b\|\bar b_{\rm log}\|_2
    &\le
    \beta \frac1n\sum_{i=1}^n
    \left|g'(Y_i \bar b_{\rm log}^\top S_i)\right|
    |Y_i|\|S_i\|_2  \\
    &\le
    \beta B,
\end{aligned}
\]
because \(|Y_i|=1\), \(\|S_i\|_2\le B\), and \(|g'(z)|\le1\). Therefore
\[
    \|\bar b_{\rm log}\|_2
    \le
    \frac{\beta B}{\rho_b}.
\]

\end{proof}

\subsection{Proof of Theorem~\ref{thm:mir-improves-validation-risk}}

\begin{proof}
By Assumption~\ref{ass:unseen-contexts}, validation prediction scores are
\(b^\top S\). Therefore validation risks depend only on the coefficient \(b\) of
the generalizable component.

For squared loss,
\[
    R_{\rm sq}(b)-R_{\rm sq}(0)
    =
    -2\mu^\top b+b^\top\Sigma b .
\]
Thus \(R_{\rm sq}(b)<R_{\rm sq}(0)\) whenever
\[
    2\mu^\top b-b^\top\Sigma b>0 .
\]
By Theorem~\ref{thm:mir-generalization},
\[
    b_{{\rm clean},{\rm sq}}^{(m)}\to0,
    ~~
    b_{{\rm MIR},{\rm sq}}^{(m)}\to\bar b_{\rm sq}.
\]
If
\[
    2\mu^\top\bar b_{\rm sq}
    -
    \bar b_{\rm sq}^\top\Sigma\bar b_{\rm sq}
    >
    0,
\]
then continuity gives
\[
    R_{\rm sq}\!\left(b_{{\rm MIR},{\rm sq}}^{(m)}\right)
    <
    R_{\rm sq}\!\left(b_{{\rm clean},{\rm sq}}^{(m)}\right)
\]
for all sufficiently large \(m\).

We next show that the squared-loss condition holds automatically when
\(\widehat\mu=\mu\) and \(\widehat\Sigma=\Sigma\). In this case,
\[
    \bar b_{\rm sq}
    =
    \beta(\rho_b I_d+\eta\Sigma)^{-1}\mu .
\]
Let
\[
    A:=(\rho_b I_d+\eta\Sigma)^{-1}.
\]
Since \(A\) and \(\Sigma\) commute,
\[
    \bar b_{\rm sq}^\top\Sigma\bar b_{\rm sq}
    =
    \beta^2\mu^\top A\Sigma A\mu
    \le
    \frac{\beta}{\eta}\,
    \beta\mu^\top A\mu
    =
    \frac{\beta}{\eta}\,
    \mu^\top\bar b_{\rm sq}.
\]
Because
\[
    \eta
    =
    \frac{\beta(\delta+\alpha)}{\delta}
    >
    \beta,
\]
we have \(\beta/\eta<1\). Also,
\[
    \mu^\top\bar b_{\rm sq}
    =
    \beta\mu^\top A\mu
    >
    0,
\]
since \(\mu\neq0\) and \(A\succ0\). Therefore
\[
    2\mu^\top\bar b_{\rm sq}
    -
    \bar b_{\rm sq}^\top\Sigma\bar b_{\rm sq}
    >
    \mu^\top\bar b_{\rm sq}
    >
    0 .
\]

For logistic loss,
\[
    \nabla R_{\rm log}(0)
    =
    -\frac12\mu .
\]
Moreover, with \(\sigma(t)=(1+e^{-t})^{-1}\), the Hessian satisfies
\[
    \nabla^2 R_{\rm log}(b)
    =
    \mathbb{E}
    \left[
        \sigma(Yb^\top S)\sigma(-Yb^\top S)SS^\top
    \right]
    \preceq
    \frac14\mathbb{E}[SS^\top]
    \preceq
    \frac{B^2}{4}I_d .
\]
Hence Taylor's expansion gives
\[
    R_{\rm log}(b)
    \le
    R_{\rm log}(0)
    -
    \frac12\mu^\top b
    +
    \frac{B^2}{8}\|b\|_2^2 .
\]
Thus \(R_{\rm log}(b)<R_{\rm log}(0)\) whenever
\[
    \mu^\top b>\frac{B^2}{4}\|b\|_2^2 .
\]
Applying this condition at \(b=\bar b_{\rm log}\), and using
Theorem~\ref{thm:mir-generalization}, gives
\[
    R_{\rm log}\!\left(b_{{\rm MIR},{\rm log}}^{(m)}\right)
    <
    R_{\rm log}\!\left(b_{{\rm clean},{\rm log}}^{(m)}\right)
\]
for all sufficiently large \(m\).

It remains to justify the stated sufficient condition for logistic loss. Let
\[
    L_{\rm emp}(b)
    :=
    \frac1n\sum_{i=1}^n g(Y_i b^\top S_i),
    ~~
    g(z):=\log(1+e^{-z}).
\]
The minimizer \(\bar b_{\rm log}\) satisfies
\[
    \rho_b\bar b_{\rm log}
    =
    -\beta\nabla L_{\rm emp}(\bar b_{\rm log}).
\]
Since \(\|S_i\|_2\le B\), the gradient \(\nabla L_{\rm emp}\) is Lipschitz on
\(\mathbb{R}^d\), and
\[
    \nabla L_{\rm emp}(0)
    =
    -\frac12\widehat\mu .
\]
Also, from the optimality equation and \(\|\nabla L_{\rm emp}(b)\|_2\le B\),
\[
    \|\bar b_{\rm log}\|_2
    \le
    \frac{\beta B}{\rho_b}.
\]
Therefore, as \(\beta/\rho_b\to0\),
\[
    \bar b_{\rm log}
    =
    \frac{\beta}{2\rho_b}\widehat\mu
    +
    o(\beta/\rho_b).
\]
If \(\mu^\top\widehat\mu>0\), then
\[
    \mu^\top\bar b_{\rm log}
    =
    \frac{\beta}{2\rho_b}\mu^\top\widehat\mu
    +
    o(\beta/\rho_b),
\]
whereas
\[
    \|\bar b_{\rm log}\|_2^2
    =
    O((\beta/\rho_b)^2).
\]
Hence, for sufficiently small \(\beta/\rho_b\),
\[
    \mu^\top\bar b_{\rm log}
    >
    \frac{B^2}{4}\|\bar b_{\rm log}\|_2^2 .
\]

The asymptotic gain results follow from the same convergence and continuity:
\[
    \Delta_{{\rm sq},m}
    \to
    R_{\rm sq}(0)-R_{\rm sq}(\bar b_{\rm sq})
    >
    0,
\]
and
\[
    \Delta_{{\rm log},m}
    \to
    R_{\rm log}(0)-R_{\rm log}(\bar b_{\rm log})
    >
    0 .
\]
\end{proof}

\subsection{Proof of Corollary~\ref{cor:hp-empirical-signal}}

\begin{proof}
Let
\[
    a:=\frac{\mu}{\|\mu\|_2}.
\]
Then
\[
    a^\top\widehat\mu
    =
    \frac1n\sum_{i=1}^n Y_i a^\top S_i .
\]
The summands satisfy
\[
    |Y_i a^\top S_i|\le B,
    ~~
    \mathbb{E}[Y_i a^\top S_i]
    =
    a^\top\mu
    =
    \|\mu\|_2 .
\]
Hoeffding's inequality \citep{vershynin2018high} gives
\[
    \mathbb{P}(a^\top\widehat\mu\le0)
     = \mathbb{P} (a^{\top}\widehat\mu - \|\mu\|_2 \le  - \|\mu\|_2  )  \le
    \exp\!\left(
        -\frac{n\|\mu\|_2^2}{2B^2}
    \right).
\]
Since \(a^\top\widehat\mu>0\) is equivalent to
\(\mu^\top\widehat\mu>0\), the result follows.
\end{proof}

\subsection{Proof of Theorem~\ref{thm:compute-dependent-gains}}

\begin{proof}
Under \(G_m=mI_n\),
\[
    T_m
    =
    \rho_w(G_m+\rho_w I_n)^{-1}
    =
    \frac{\rho_w}{m+\rho_w}I_n .
\]
Write
\[
    t_m:=\frac{\rho_w}{m+\rho_w}.
\]
Using the profiled squared-loss formula from the proof of
Theorem~\ref{thm:mir-generalization}, and using
\(\widehat\mu=\mu\) and \(\widehat\Sigma=\Sigma\), we obtain
\[
    b_{{\rm clean},{\rm sq}}^{(m)}
    =
    t_m(\rho_b I_d+t_m\Sigma)^{-1}\mu .
\]
For the key-ablating objective, profiling out \(u\) gives
\[
    \frac1{2n}
    (Y-\mathbf S b)^\top T_m(Y-\mathbf S b)
    +
    \frac{\beta}{2n}\|Y-\mathbf S b\|_2^2
    +
    \frac{\rho_b}{2}\|b\|_2^2 .
\]
Differentiating with respect to \(b\) yields
\[
    b_{{\rm key},{\rm sq}}^{(m)}
    =
    (t_m+\beta)
    (\rho_b I_d+(t_m+\beta)\Sigma)^{-1}\mu .
\]

Let
\[
    \Sigma=U{\rm diag}(\lambda_1,\ldots,\lambda_d)U^\top,
    ~~
    U^\top\mu=(\mu_1,\ldots,\mu_d)^\top .
\]
If \(\lambda_j=0\), then \(\mu_j=0\) since $\mu = \EE[YS]$ and $\Sigma=\EE[SS^{\top}]$. Indeed, \(\lambda_j=0\) implies that the
corresponding projection of \(S\) is zero almost surely, and hence its
correlation with \(Y\) is also zero. Thus only terms with \(\lambda_j>0\)
contribute to the risk.

For \(\alpha_0>0\), define
\[
    b(\alpha_0):=
    \alpha_0(\rho_b I_d+\alpha_0\Sigma)^{-1}\mu .
\]
In the eigenbasis of \(\Sigma\), the \(j\)-th coordinate is, for
\(\lambda_j>0\),
\[
    b_j(\alpha_0)
    =
    \frac{\mu_j}{\lambda_j}
    \frac{\alpha_0\lambda_j/\rho_b}{1+\alpha_0\lambda_j/\rho_b}.
\]
Let
\[
    s(x):=\frac{x}{1+x}.
\]
The reduction in squared risk from using \(b(\alpha_0)\) instead of \(0\) is
\[
\begin{aligned}
    R_{\rm sq}(0)-R_{\rm sq}(b(\alpha_0))
    &=
    2\mu^\top b(\alpha_0)-b(\alpha_0)^\top\Sigma b(\alpha_0).
\end{aligned}
\]
Write \(U^\top b(\alpha_0)=(b_1(\alpha_0),\ldots,b_d(\alpha_0))^\top\). In the
eigenbasis of \(\Sigma\),
\[
    \mu^\top b(\alpha_0)
    =
    \sum_{j=1}^d \mu_j b_j(\alpha_0),
    \qquad
    b(\alpha_0)^\top\Sigma b(\alpha_0)
    =
    \sum_{j=1}^d \lambda_j b_j(\alpha_0)^2 .
\]
Therefore
\[
\begin{aligned}
    R_{\rm sq}(0)-R_{\rm sq}(b(\alpha_0))
    &=
    \sum_{j=1}^d
    \left\{
        2\mu_j b_j(\alpha_0)
        -
        \lambda_j b_j(\alpha_0)^2
    \right\}.
\end{aligned}
\]
If \(\lambda_j=0\), then \(\mu_j=0\), so the corresponding term is zero. Thus
only the terms with \(\lambda_j>0\) remain. For such \(j\),
\[
    b_j(\alpha_0)
    =
    \frac{\mu_j}{\lambda_j}
    s\!\left(\frac{\alpha_0\lambda_j}{\rho_b}\right),
    \qquad
    s(x):=\frac{x}{1+x}.
\]
Substituting this expression gives
\[
\begin{aligned}
    2\mu_j b_j(\alpha_0)
    -
    \lambda_j b_j(\alpha_0)^2
    &=
    2\mu_j
    \frac{\mu_j}{\lambda_j}
    s\!\left(\frac{\alpha_0\lambda_j}{\rho_b}\right)
    -
    \lambda_j
    \left[
        \frac{\mu_j}{\lambda_j}
        s\!\left(\frac{\alpha_0\lambda_j}{\rho_b}\right)
    \right]^2 \\
    &=
    \frac{2\mu_j^2}{\lambda_j}
    s\!\left(\frac{\alpha_0\lambda_j}{\rho_b}\right)
    -
    \frac{\mu_j^2}{\lambda_j}
    s^2\!\left(\frac{\alpha_0\lambda_j}{\rho_b}\right) \\
    &=
    \frac{\mu_j^2}{\lambda_j}
    s\!\left(\frac{\alpha_0\lambda_j}{\rho_b}\right)
    \left[
        2-
        s\!\left(\frac{\alpha_0\lambda_j}{\rho_b}\right)
    \right].
\end{aligned}
\]
Hence
\[
    R_{\rm sq}(0)-R_{\rm sq}(b(\alpha_0))
    =
    \sum_{\lambda_j>0}
    \frac{\mu_j^2}{\lambda_j}
    s\!\left(\frac{\alpha_0\lambda_j}{\rho_b}\right)
    \left[
        2-
        s\!\left(\frac{\alpha_0\lambda_j}{\rho_b}\right)
    \right].
\]
Since
\[
    b_{{\rm clean},{\rm sq}}^{(m)}=b(t_m),
    ~~
    b_{{\rm key},{\rm sq}}^{(m)}=b(t_m+\beta),
\]
the gain \(\Delta_{{\rm key},m}\) is a sum over \(\lambda_j>0\) of terms of the
form
\[
    \frac{\mu_j^2}{\lambda_j}
    \left[
        s(x+\kappa_j)\{2-s(x+\kappa_j)\}
        -
        s(x)\{2-s(x)\}
    \right],
\]
where
\[
    x=\frac{t_m\lambda_j}{\rho_b},
    ~~
    \kappa_j=\frac{\beta\lambda_j}{\rho_b}>0 .
\]
Note that
\[
    s(x+\kappa)-s(x)
    =
    \frac{\kappa}{(1+x)(1+x+\kappa)}
\]
and
\[
    2-s(x+\kappa)-s(x)
    =
    \frac{\kappa+2x+2}{(1+x)(1+x+\kappa)} .
\]
We have that each nonzero spectral contribution equals
\[
    \frac{\mu_j^2}{\lambda_j}
    \frac{
        \kappa_j(\kappa_j+2x+2)
    }{
        (1+x)^2(1+x+\kappa_j)^2
    } .
\]
For fixed \(\kappa>0\), define
\[
    F_\kappa(x)
    :=
    \frac{
        \kappa(\kappa+2x+2)
    }{
        (1+x)^2(1+x+\kappa)^2
    } .
\]
Then
\[
    F_\kappa'(x)
    =
    -
    \frac{
        2\kappa
        \left(
            \kappa^2+3\kappa x+3\kappa+3x^2+6x+3
        \right)
    }{
        (1+x)^3(1+x+\kappa)^3
    }
    <0 .
\]
Since
\[
    t_m=\frac{\rho_w}{m+\rho_w}
\]
is strictly decreasing in \(m\), each nonzero spectral contribution to
\(\Delta_{{\rm key},m}\) is strictly increasing in \(m\). Because
\(\mu\neq0\), at least one such contribution is nonzero. Hence
\(\Delta_{{\rm key},m}\) is strictly increasing in \(m\).

Finally, \(t_m\to0\), so \(x\to0\) for every \(j\), and
\[
    \lim_{m\to\infty}\Delta_{{\rm key},m}
    =
    \sum_{\lambda_j>0}
    \frac{\mu_j^2}{\lambda_j}
    \frac{\kappa_j(\kappa_j+2)}{(1+\kappa_j)^2}
    >
    0 .
\]
\end{proof}

\section{Derivation of Quanta Scaling Law}
\label{app:quanta-derivation}
To make the paper self-contained and easier for readers to follow, this section summarizes
the Quanta argument from \citet{quanta-scalinglaw} that we use in our paper. The background is skill learning: next-token
prediction is assumed to require a large collection of discrete predictive
skills, called \emph{quanta}. A model either learns a quantum or it does not,
and scaling improves performance by allowing the model to learn more quanta in
descending order of usefulness.

\textbf{Loss as a Function of Learned Quanta.}
Index quanta by decreasing use frequency. Let \(p_k\) be the probability that
the \(k\)-th quantum is needed on a randomly drawn token. The Quanta model
assumes a Zipf tail
\begin{equation}
    p_k = \frac{1}{Z} k^{-(1+\alpha)},
    \qquad
    Z = \sum_{k=1}^{\infty} k^{-(1+\alpha)},
    \qquad
    \alpha > 0.
\end{equation}
In the simplest monogenic version of the model, each token depends mainly on
one quantum. Suppose learning a quantum lowers the loss on those tokens from
\(b\) to \(a\), with \(b>a\). If the model has learned the first \(n\) quanta,
its expected loss is
\begin{align}
    L(n)
    = \sum_{k=1}^{n} a p_k
    +
    \sum_{k=n+1}^{\infty} b p_k =
    a + (b-a)\sum_{k=n+1}^{\infty} p_k \approx
    a + \frac{b-a}{\alpha Z} n^{-\alpha},
\end{align}
where the last line uses the standard tail approximation
\[\sum_{k=n+1}^{\infty} k^{-(1+\alpha)} \approx \int_n^\infty x^{-(1+\alpha)}dx
= n^{-\alpha}/\alpha.\]
Therefore
\begin{equation}
    L(n) \approx E + C n^{-\alpha},
    \label{eq:app-quanta-n}
\end{equation}
where \(E=a\) is the irreducible loss floor and \(C>0\) absorbs the remaining
constants. This is the key step: a Zipf distribution over skill frequencies
induces a power law in the loss as a function of the number of learned skills.

\textbf{Parameter Scaling.}
If data is abundant, then the bottleneck is model capacity. Assume each quantum
requires approximately \(c_N\) parameters to represent. A model with \(N\)
parameters can then learn
\begin{equation}
    n_N \approx \frac{N}{c_N}
\end{equation}
quanta. Substituting this into Eq.~\eqref{eq:app-quanta-n} gives
\begin{equation}
    L(N,\infty)
    \approx
    E + A_N N^{-\alpha},
\end{equation}
so the parameter-scaling exponent is
\begin{equation}
    \alpha_N = \alpha.
\end{equation}

\textbf{Data Scaling.}
In the data-constrained multi-epoch regime, repeated passes over the same corpus
do not create new rare skills. The relevant resource is the number of unique
tokens \(U\). Assume that learning the \(k\)-th quantum requires at least
\(\tau\) tokens in the unique dataset that use that quantum. Then the last
quantum that can be learned, denoted \(n_U\), satisfies
\begin{equation}
    U p_{n_U} \approx \tau.
\end{equation}
Using \(p_k \propto k^{-(1+\alpha)}\), we obtain
\begin{equation}
    n_U
    \approx
    \left(\frac{U}{Z\tau}\right)^{1/(1+\alpha)}.
\end{equation}
Substituting again into Eq.~\eqref{eq:app-quanta-n} yields
\begin{equation}
    L(\infty,U)
    \approx
    E + A_U U^{-\alpha/(1+\alpha)},
\end{equation}
so the data-scaling exponent is
\begin{equation}
    \alpha_U = \frac{\alpha}{1+\alpha}.
\end{equation}
This is why the data exponent is smaller than the parameter exponent in the
basic Quanta picture.

\textbf{A Quanta-Motivated Joint Law.}
The single-axis derivations above do not uniquely determine a joint
\((N,U)\) law, but they do imply that the number of learned quanta is jointly
limited by parameter capacity and unique-data coverage. A hard bottleneck view
would write
\begin{equation}
    n(N,U)
    \lesssim
    \min\left\{
        \gamma_N N,\,
        \gamma_U U^{1/(1+\alpha)}
    \right\},
\end{equation}
for some constants \(\gamma_N,\gamma_U>0\). Equivalently,
\begin{equation}
    n(N,U)^{-1}
    \gtrsim
    \max\left\{
        \frac{1}{\gamma_N N},\,
        \frac{1}{\gamma_U U^{1/(1+\alpha)}}
    \right\}.
\end{equation}
For fitting, it is convenient to replace this hard maximum by a smooth additive
envelope in inverse-skill space,
\begin{equation}
    n(N,U)^{-1}
    \approx
    \frac{A'}{N}
    +
    \frac{B'}{U^{1/(1+\alpha)}}.
\end{equation}
Substituting this into \(L-E \propto n^{-\alpha}\) gives the Quanta-motivated
joint law
\begin{equation}
    L_{\mathrm{Q}}(N,U)
    =
    E+
    \left(
        \frac{A}{N}
        +
        \frac{B}{U^{1/(1+\alpha)}}
    \right)^{\alpha},
    \label{eq:app-quanta-derived}
\end{equation}
which is exactly the form used in Eq.~\eqref{eq:app-quanta}. This coupling
should be read as a smooth interpolation motivated by the Quanta asymptotes. It
is attractive because it recovers both derived limits:
\begin{align}
    L_{\mathrm{Q}}(N,\infty)
    &= E + A^{\alpha} N^{-\alpha}, \\
    L_{\mathrm{Q}}(\infty,U)
    &= E + B^{\alpha} U^{-\alpha/(1+\alpha)}.
\end{align}
Thus, the Quanta picture explains why the benefit of increasing model size
should depend on the available unique data: both resources control the number of
skills that can be learned, and the loss is governed by that shared latent
quantity.


\end{document}

%% file: acknowledgements.tex
\section{Acknowledgments}
We thank Eric Czech, Hrayr Harutyunyan, and Samip Dahal for helpful discussions and their invaluable feedback. This work was supported in part by funding from the DARPA AIQ program, the Office of Naval Research under grant N00014-23-1-2590, the National Science Foundation under grant No. 2310831, No. 2428059, No. 2435696, No. 2440954, a Michigan Institute for Data Science Propelling Original Data Science (PODS) grant, Two Sigma Investments LP, and  LG Management Development Institute AI Research.